\documentclass[10pt,journal,compsoc]{IEEEtran}

\ifCLASSOPTIONcompsoc
  
  \usepackage[nocompress]{cite}
\else
  
  \usepackage{cite}
\fi

\ifCLASSINFOpdf
  \usepackage[pdftex]{graphicx}  
\else

\fi

\usepackage{amsmath,graphicx}
\usepackage{amsthm}
\usepackage{amssymb}
\usepackage{epstopdf}
\usepackage{subfigure}
\usepackage{color}
\usepackage{caption}
\usepackage{mathtools}
\usepackage{tabulary}
\usepackage{booktabs}
\usepackage{xr}
\usepackage{zref-base}
\usepackage{atveryend}
\makeatletter
\AfterLastShipout{%
  \if@filesw   
    \begingroup
      \zref@setcurrent{default}{\the\value{equation}}%
      \zref@wrapper@immediate{%
        \zref@labelbyprops{eq-last}{default}%
      }%
    \endgroup
  \fi
}
\makeatother

\newtheorem{theorem}{Theorem}[section]

\newtheorem{lemma}[theorem]{Lemma}
\newtheorem*{lemma*}{Lemma}

\newtheorem{noiselessResult}{Result}[section]

\begin{document}

\title{Clustering of Data with Missing Entries using Non-convex Fusion Penalties}

\author{Sunrita~Poddar,~\IEEEmembership{Student Member,~IEEE,}
               and~Mathews~Jacob,~\IEEEmembership{Senior~Member,~IEEE}
\IEEEcompsocitemizethanks{\IEEEcompsocthanksitem S. Poddar and M.  Jacob are with the Department
of Electrical and Computer Engineering, University of Iowa, Iowa City,
IA, 52246.\protect\\
}}

\IEEEtitleabstractindextext{
\begin{abstract}
The presence of missing entries in data often creates challenges for pattern recognition algorithms. Traditional algorithms for clustering data assume that all the feature values are known for every data point. We propose a method to cluster data in the presence of missing information. Unlike conventional clustering techniques where every feature is known for each point, our algorithm can handle cases where a few feature values are unknown for every point. For this more challenging problem, we provide theoretical guarantees for clustering using a $\l_0$ fusion penalty based optimization problem. Furthermore, we propose an algorithm to solve a relaxation of this problem using saturating non-convex fusion penalties. It is observed that this algorithm produces solutions that degrade gradually with an increase in the fraction of missing feature values. We demonstrate the utility of the proposed method using a simulated dataset, the Wine dataset and also an under-sampled cardiac MRI dataset. It is shown that the proposed method is a promising clustering technique for datasets with large fractions of missing entries.
\end{abstract}}

\maketitle

\IEEEdisplaynontitleabstractindextext

\IEEEpeerreviewmaketitle

\IEEEraisesectionheading{\section{Introduction}\label{sec:introduction}}

\IEEEPARstart{C}{lustering} is an exploratory data analysis technique that is widely used to discover natural groupings in large datasets,  where no labeled or pre-classified samples are available apriori. Specifically, it assigns an object to a group if it is similar to other objects within the group, while being dissimilar to objects in other groups. Example applications include analysis of gene experssion data, image segmentation, identification of lexemes in handwritten text,  search result grouping and recommender systems \cite{saxena2017review}. A wide variety of clustering methods have been introduced over the years; see \cite{jain1999data} for a review of classical methods. However, there is no consensus on a particular clustering technique that works well for all tasks, and there are pros and cons to most existing algorithms. The common clustering techniques such as k-means \cite{macqueen1967some}, k-medians \cite{bradley1997clustering} and spectral clustering \cite{ng2001spectral} are implemented using the Lloyd's algorithm which is non-convex and thus sensitive to initialization. Recently, linear programming and semi-definite programming based convex relaxations of the k-means and k-medians algorithms have been introduced \cite{awasthi2015relax} to overcome the dependence on initialization. Unlike the Lloyd's algorithm, these relaxations can provide a certificate of optimality. However, all of the above mentioned techniques require apriori knowledge of the desired number of clusters. Hierarchical clustering methods \cite{ward1963hierarchical}, which produce easily interpretable and visualizable clustering results for a varying number of clusters, have been introduced to overcome the above challenge. A drawback of \cite{ward1963hierarchical} is its sensitivity to initial guess and perturbations in the data. The more recent convex clustering technique (also known as sum-of-norms clustering) \cite{hocking2011clusterpath} retains the advantages of hierarchical clustering, while being invariant to initialization, and producing a unique clustering path. Theoretical guarantees for successful clustering using the convex-clustering technique are also available \cite{zhu2014convex}. 

Most of the above clustering algorithms cannot be directly applied to real-life datasets, where a large fraction of samples are missing. For example, gene expression data often contains missing entries due to image corruption, fabrication errors or contaminants \cite{de2015impact}, rendering gene cluster analysis difficult. Likewise, large databases used by recommender systems (e.g Netflix) usually have a huge amount of missing data, which makes pattern discovery challenging \cite{bell2008bellkor}. The presence of missing responses in surveys \cite{brick1996handling} and failing imaging sensors in astronomy \cite{wagstaff2005making} are reported to make the analysis in these applications challenging. Several approaches were introduced to extend clustering to missing-data applications. For example, a partially observed dataset can be converted to a fully observed one using either deletion or imputation \cite{dixon1979pattern}. Deletion involves removal of variables with missing entries, while imputation tries to estimate the missing values and then performs clustering on the completed dataset. An extension of the weighted sum-of-norms algorithm (originally introduced for fully sampled data \cite{hocking2011clusterpath}) has been proposed where  the weights are estimated from the data points by using some imputation techniques on the missing entries \cite{chen2015convex}. Kernel-based methods for clustering have also been extended to deal with missing entries by replacing Euclidean distances with  partial distances \cite{hathaway2001fuzzy,sarkar2001fuzzy}. 
A majorize minimize algorithm was introduced to solve for the cluster-centres and cluster memberships in \cite{chi2016k}, which offers proven reduction in cost with iteration. In \cite{hunt2003mixture} and \cite{lin2006fast} the data points are assumed to lie on a mixture of $K$ distributions, where $K$ is known. The algorithms then alternate between the maximum likelihood estimation of the distribution parameters and the missing entries. A challenge with these algorithms is the lack of theoretical guarantees for successful clustering in the presence of missing entries. In contrast, there has been a lot of work in recent years on matrix completion for different data models. Algorithms along with theoretical guarantees have been proposed for low-rank matrix completion \cite{candes2009exact} and subspace clustering from data with missing entries \cite{DBLP:journals/corr/abs-1112-5629}, \cite{elhamifar2016high}. However, these  algorithms and their theoretical guarantees cannot be trivially extended to  the problem of clustering in the presence of missing entries. 

The main focus of this paper is to introduce an algorithm for the clustering of data with missing entries and to theoretically analyze the conditions needed for perfect clustering in the presence of missing data. The proposed algorithm is inspired by the sum-of-norms clustering technique \cite{hocking2011clusterpath}; it is formulated as an optimization problem, where an auxiliary variable assigned to each data point is an estimate of the centre of the cluster to which that point belongs. A fusion penalty is used to enforce equality between many of these auxiliary variables. Since we have experimentally observed that non-convex fusion penalties provide superior clustering performance, we focus on the analysis of clustering using a $\ell_0$ fusion penalty in the presence of missing entries, for an arbitrary number of clusters. The analysis reveals that perfect clustering is guaranteed with high probability, provided the number of measured entries (probability of sampling) is high enough; the required number of measured entries depends on several parameters including intra-cluster variance and inter-cluster distance. We observe that the required number of entries is critically dependent on coherence, which is a measure of the concentration of inter cluster differences in the feature space. Specifically,  if the clustering of the points is determined only by a very small subset of all the available features, then the clustering becomes quite unstable if those particular feature values are unknown for some points. Other factors which influence the clustering technique are the number of features, number of clusters and total number of points. We also extend the theoretical analysis to the case without missing entries. The analysis in this setting shows improved bounds when a uniform random distribution of points in their respective clusters is considered, compared to the worst case analysis considered in the missing-data setting. We expect that improved bounds can also be derived for the case with missing data when a uniform random distribution is considered.

We also propose an algorithm to solve a relaxation of the above $\ell_0$ penalty based clustering problem, using non-convex saturating fusion penalties. The algorithm is demonstrated on a simulated dataset with different fractions of missing entries and cluster separations. We observe that the algorithm is stable with changing fractions of missing entries, and the clustering performance degrades gradually with an increase in the number of missing entries. We also demonstrate the algorithm on clustering of the Wine dataset \cite{Lichman:2013} and reconstruction of a dynamic cardiac MRI dataset from few Fourier measurements.

\section{Clustering using $\ell_0$ fusion penalty}

\subsection{Background}
We consider the clustering of points drawn from one of $K$ distinct clusters $C_1, C_2, \ldots, C_K$. We denote the center of the clusters by $\mathbf c_1, \mathbf c_2, \ldots, \mathbf c_K \in \mathbb R^P$. For simplicity, we assume that there are $M$ points in each of the clusters. The individual points in the $k^{\rm th}$ cluster are modelled as:
\begin{equation}
\label{noisemodel}
\mathbf z_k(m) = \mathbf c_k + \mathbf n_{k}(m); ~~m=1,..,M, ~k=1,\ldots,K
\end{equation}
Here, $\mathbf n_{k}(m)$ is the noise or the variation of $\mathbf z_k(m)$ from the cluster center $\mathbf c_k$. The set of input points $\{\mathbf x_i\},i=1,..,KM$ is obtained as a random permutation of the points $\{\mathbf z_k(m)\}$. The objective of a clustering algorithm is to estimate the cluster labels, denoted by $\mathcal C(\mathbf x_i)$ for $i = 1,..,KM$. 

The sum-of-norms (SON) method is a recently proposed convex clustering algorithm \cite{hocking2011clusterpath}. Here, a surrogate variable $\mathbf u_i$ is introduced for each point $\mathbf x_i$, which is an estimate of the centre of the cluster to which $\mathbf x_i$ belongs. As an example, let $K=2$ and $M=5$. Without loss of generality, let us assume that $\mathbf x_1, \mathbf x_2, \ldots, \mathbf x_5$ belong to $\mathcal C_1$ and  $\mathbf x_6, \mathbf x_7, \ldots, \mathbf x_{10}$ belong to $\mathcal C_2$. Then, we expect to arrive at the solution: $\mathbf u_1= \mathbf u_2 = \ldots = \mathbf u_5= \mathbf c_1$ and $\mathbf u_6= \mathbf u_7 = \ldots = \mathbf u_{10} = \mathbf c_2$. In order to find the optimal $\{\mathbf u_i^*\}$, the following optimization problem is solved:

\begin{equation}
\label{SON}
\{\mathbf u_i^*\} = \arg \min_{\{\mathbf u_i\}} \sum_{i=1}^{KM}\|\mathbf x_i - \mathbf u_i\|_2^2 + \lambda \sum_{i=1}^{KM} \sum_{j=1}^{KM} \|\mathbf u_i - \mathbf u_j\|_{p}
\end{equation}
The fusion penalty ($\|\mathbf u_i - \mathbf u_j\|_{p}$) can be enforced using different $\ell_p$ norms, out of which the $\ell_1$, $\ell_2$ and $\ell_\infty$ norms have been used in literature \cite{hocking2011clusterpath}. The use of sparsity promoting fusion penalties encourages sparse differences $\mathbf u_i-\mathbf u_j$, which facilitates the clustering of the points $\{\mathbf u_i\}$. For an appropriately chosen $\lambda$, the $\mathbf u_i$'s corresponding to $\mathbf x_i$'s from the same cluster converge to the same point. The main benefit of this convex scheme over classical clustering algorithms is the convergence of the algorithm to the global minimum.

The above optimization problem can be solved efficiently using the Alternating Direction Method of Multipliers (ADMM) algorithm and the Alternating Minimization Algorithm (AMA)  \cite{chi2015splitting}. Truncated $\ell_1$ and $\ell_2$ norms have also been used recently in the fusion penalty, resulting in non-convex optimization problems \cite{pan2013cluster}. It has been shown that these penalties provide superior performance to the traditional convex penalties. Convergence to local minimum using an iterative algorithm has also been guaranteed in the non-convex setting.

The sum-of-norms algorithm has also been used as a visualization and exploratory tool to discover patterns in datasets \cite{chen2015convex}. Clusterpath diagrams are a common way to visualize the data. This involves plotting the solution path as a function of the regularization parameter $\lambda$. For a very small value of $\lambda$, the solution is given by: $\mathbf u_i^* = \mathbf x_i$, i.e. each point forms its individual cluster. For a very large value of $\lambda$, the solution is given by: $\mathbf u_i^* = c$, i.e. every point belongs to the same cluster. For intermediate values of $\lambda$, more interesting behaviour is seen as various $\{\mathbf u_i\}$ merge and reveal the cluster structure of the data. 

In this paper, we extend the algorithm to account for missing entries in the data. We present theoretical guarantees for clustering with and without missing entries using an $\ell_0$ fusion penalty. Next, we approximate the $\ell_0$ penalty by non-convex saturating penalties, and solve the resulting relaxed optimization problem using an iterative reweighted least squares (IRLS) strategy \cite{chartrand2008iteratively}. The proposed algorithm is shown to perform clustering correctly in the presence of large fractions of missing entries. 

\begin{figure}[!t]
	\centering
	\center{\includegraphics[width=0.5\textwidth]{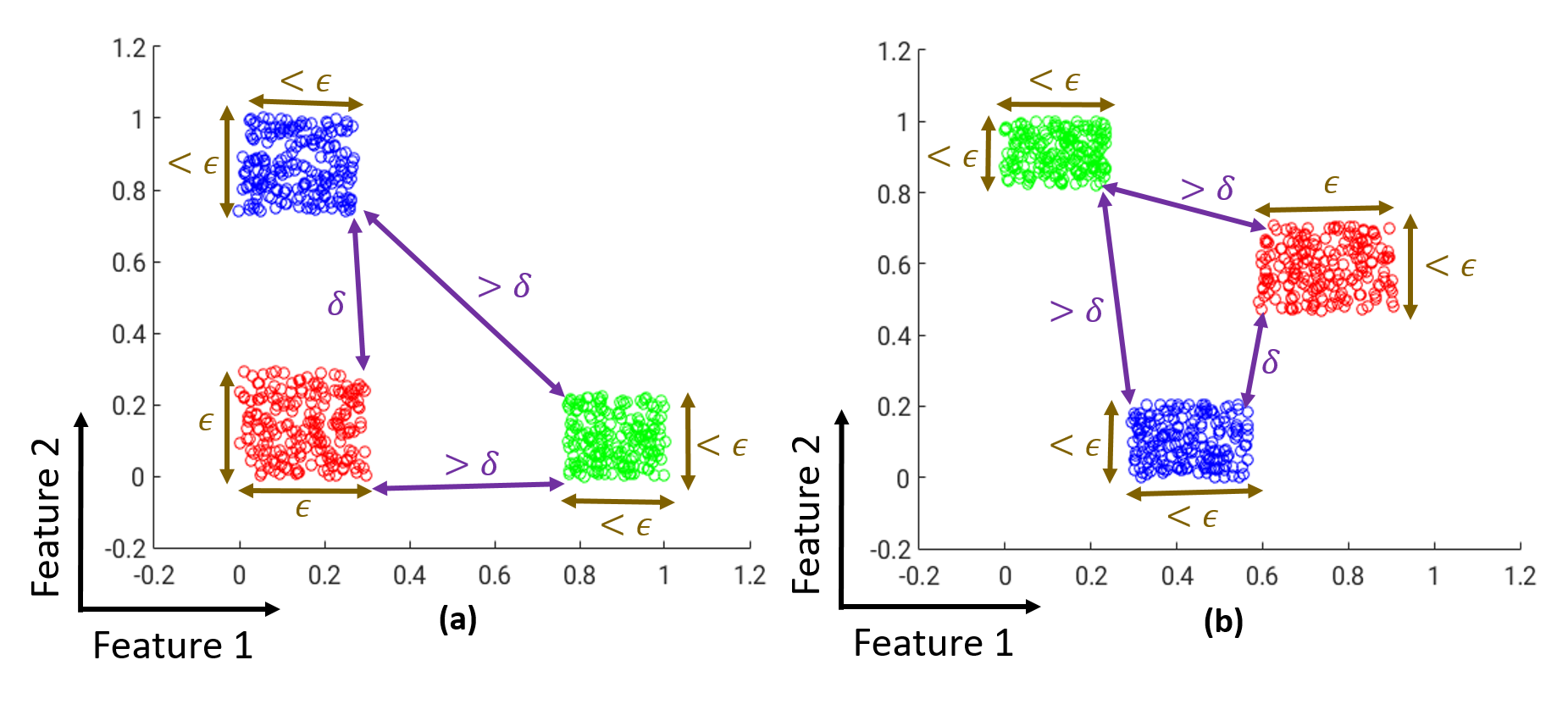}}
	\caption{Central Assumptions: (a) and (b) illustrate different  instances where points belonging to $\mathbb{R}^2$ are to be separated into 3 different clusters (denoted by the colours red, green and blue). Assumptions A.1 and A.2 related to cluster separation and cluster size respectively, are illustrated in both (a) and (b). The importance of assumption A.3 related to feature concentration can also be appreciated by comparing (a) and (b). In (a), points in the red and blue clusters cannot be distinguished solely on the basis of feature 1, while the red and green clusters cannot be distinguished solely on the basis of feature 2. Thus, it is difficult to correctly  cluster these points if either of the feature values is unknown. In (b), due to low coherence (as assumed in A.3), this problem does not arise.}
	\label{FigAssump}
\end{figure}

\subsection{Central Assumptions}
We make the following assumptions (illustrated in Fig \ref{FigAssump}), which are key to the successful clustering of the points: 
\begin{description}
\item [A.1:] \textbf{Cluster separation:} Points from different clusters are separated by $\delta >0$ in the $\ell_2$ sense, i.e:
\begin{equation}
\label{deltadef}
\min_{\{m,n\}}\|\mathbf z_{k}(m) -\mathbf z_{l}(n)\|_{2} \geq \delta; ~\forall\; k\neq l
\end{equation}
We require $\delta >0$ for the clusters to be non-overlapping. A high $\delta$ corresponds to well separated clusters.
\item [A.2:] \textbf{Cluster size:} The maximum separation of points within any cluster in the $\ell_{\infty}$ sense is $\epsilon \geq 0$, i.e:
\begin{equation}
\label{epsdef}
\max_{\{m,n\}}\|\mathbf z_{k}(m) - \mathbf z_{k}(n)\|_{\infty} = \epsilon; ~\forall k=1,\ldots,K
\end{equation}
Thus, the $k^{\rm th}$ cluster is contained within a cube of size $\epsilon$, with center $\mathbf c_k$.
\item [A.3:] \textbf{Feature concentration:} The coherence of a vector $\mathbf y \in \mathbb R^P$ is defined as \cite{candes2009exact}:
\begin{equation}
\mu(\mathbf y) = \frac{P\|\mathbf y\|_{\infty}^2}{\|\mathbf y\|_2^2}
\end{equation}
By definition: $1 \leq \mu(\mathbf y) \leq P$. Intuitively, a vector with a high coherence has a few large values and several small ones. Specifically, if $\mu(\mathbf y) = P$, then $\mathbf y$ has only $1$ non-zero value. In contrast, if $\mu(\mathbf y) = 1$, then all the entries of $\mathbf y$ are equal. We bound the coherence of the difference between points from different clusters as:
\begin{equation}
\label{coherence}
\max_{\{m,n\}}\mu(\mathbf z_k(m) - \mathbf z_l(n)) \leq \mu_0; ~\forall\; k\neq l
\end{equation}
$\mu_0$ is indicative of the difficulty of the clustering problem in the presence of missing data. If $\mu_0=P$, then two clusters differ only a single feature,  suggesting that it is difficult to assign the correct cluster to a point if this feature is not sampled. The best case scenario is $\mu_0=1$, when all the features are equally important. In general, cluster recovery from missing data becomes challenging with increasing $\mu_0$.
\end{description}
The quantity $\kappa = \frac{\epsilon\sqrt P}{\delta}$ is a measure of the difficulty of the clustering problem. Small values of $\kappa$ suggest large inter-cluster separation compared to the cluster size; the recovery of such well-defined clusters is expected to be easier than the case with large $\kappa$ values. Note the $\ell_2$ norm is used in the definition of $\delta$, while the $\ell_{\infty}$ norm is used to define $\epsilon$. If $\delta = \epsilon\sqrt P$, then $\kappa = 1$; this value of $\kappa$ is of special importance since $\kappa <  1$ is a requirement for successful recovery in our main results. 

We study the problem of clustering the points $\{\mathbf x_i\}$ in the presence of entries missing uniformly at random. We arrange the points $\{\mathbf x_i\}$ as columns of a matrix $\mathbf X$. The rows of the matrix are referred to as features. We assume that each entry of $\mathbf X$ is observed with probability $p_0$. The entries measured in the $i^{th}$ column are denoted by:
\begin{equation}
\mathbf y_i = \mathbf S_i\, \mathbf x_i, ~~ i=1,..,KM
\end{equation}
where $\mathbf S_i$ is the sampling matrix, formed by selecting rows of the identity matrix. 
We consider solving the following optimization problem to obtain the cluster memberships from data with missing entries:
\begin{equation}
\label{l0prob}
\begin{split}
\{\mathbf u_i^{*}\} = & \min_{\{\mathbf u_i\}} \sum_{i=1}^{KM}\sum_{j=1}^{KM}\|\mathbf u_i - \mathbf u_j\|_{2,0}\\ & \mbox{ s.t } \|\mathbf S_i\;(\mathbf x_i - \mathbf u_i)\|_\infty \leq {\frac{\epsilon}{2}}, i\in\{1 \ldots KM\}
\end{split}
\end{equation}
We use the above constrained formulation rather than the unconstrained formulation in \eqref{SON} to avoid the dependence on $\lambda$. The $\ell_{2,0}$ norm is defined as:
\begin{equation}
\|\mathbf x\|_{2,0} = \begin{cases}
0 &, \text{if $\|\mathbf x\|_2 = 0$}\\
1 &, \text{otherwise}
\end{cases}
\end{equation}
Similar to the SON scheme \eqref{SON}, we expect that all $\mathbf u_i$'s that correspond to $\mathbf x_i$ in the same cluster are equal, while $\mathbf u_i$'s from different clusters are not equal. We consider the cluster recovery to be successful when there are no mis-classifications. We claim that the above algorithm can successfully recover the clusters with high probability when: 
\begin{enumerate}
\item  The clusters are well separated (i.e, low $\kappa = \frac{\epsilon\sqrt{P}}{\delta})$).
\item  The sampling probability $p_0$ is sufficiently high.
\item  The coherence $\mu_0$ is small.
\end{enumerate}

Before moving on to a formal statement and proof of this result, we consider a simple special case to illustrate the approach. In order to aid the reader in following the results, all the important symbols used in the paper have been summarized in Table \ref{tab:notUsed}.

\begin{table}[htbp]\caption{Notations used}
\begin{center}
\begin{tabular}{r c p{5.7cm} }
\toprule
&$K$ & Number of clusters\\
& $M$ & Number of points in each cluster\\
& $P$ & Number of features for each point\\
&$\mathcal C_i$ & The $i^{th}$ cluster\\
& $\mathbf c_i$ & Centre of $\mathcal C_i$\\
& $\mathbf z_i(m)$ & $m^{th}$ point in $\mathcal C_i$\\
& $\{\mathbf x_i\}$ & Random permutation of $KM$ points $\{\mathbf z_k(m)\}$ for $k \in \{1,2,\ldots, K\}, m \in \{1,2,\ldots, M\}$\\
& $\mathbf S_i$ & Sampling matrix for $\mathbf x_i$\\
& $\mathbf X$ & Matrix formed by arranging $\{\mathbf x_i\}$ as columns, such that the $i^{th}$ column is $\mathbf x_i$\\
& $p_0$ & Probability of sampling each entry in $\mathbf X$\\
& $\delta$ & Parameter related to cluster separation defined in \eqref{deltadef}\\
& $\epsilon$ & Parameter related to cluster size defined in \eqref{epsdef}\\
& $\kappa$ & Defined as $\kappa = \frac{\epsilon \sqrt{P}}{\delta}$\\
& $\mu_0$ & Parameter related to coherence defined in \eqref{coherence}\\
& $\gamma_0$ & Defined in \eqref{gamma0}\\
& $\delta_0$ & Defined in \eqref{delta0}\\
& $\beta_0$ & Defined in \eqref{beta0}\\
& $\eta_0$ & Defined in \eqref{eta0}\\
& $\eta_{0,{\rm approx}}$ & Upper bound for $\eta_0$ for the case of 2 clusters, defined in \eqref{etaApprox}\\
& $c$ & Parameter related to cluster centre separation defined in \eqref{cDef}\\
& $\kappa'$ & Defined as $\kappa' = \frac{\epsilon \sqrt{P}}{c}$\\
& $\beta_1$ & Defined in \eqref{beta1}\\
& $\eta_1$ & Probability of failure of Theorem \ref{noMissingUnifFinal}\\
\bottomrule
\end{tabular}
\end{center}
\label{tab:notUsed}
\end{table}

\subsection{Noiseless Clusters with Missing Entries}

We consider the simple case where all the points belonging to the same cluster are identical. Thus every cluster is "noiseless", and we have: $\epsilon=0$ and hence $\kappa = 0$. The optimization problem \eqref{l0prob} now reduces to:

\begin{equation}
\label{l0probNoiseless}
\begin{split}
\{\mathbf u_i^{*}\} = & \min_{\{\mathbf u_i\}} \sum_{i=1}^{KM}\sum_{j=1}^{KM}\|\mathbf u_i - \mathbf u_j\|_{2,0}\\ & \mbox{ s.t } \mathbf S_i\,\mathbf x_i = \mathbf S_i\, \mathbf u_i, i\in\{1 \ldots KM\}
\end{split}
\end{equation}
Next, we state a few results for this special case in order to provide some intuition about the problem. The results are not stated with mathematical rigour and are not accompanied by proofs. In the next sub-section, when we consider the general case, we will provide lemmas and theorems (with proofs in the appendix), which generalize the results stated here. Specifically, Lemmas \ref{cent}, \ref{partDist}, \ref{smallClusters} and Theorem \ref{mainResult} generalize Results \ref{noiselessLemmA.0}, \ref{noiselessLemmA.1}, \ref{noiselessLemmA.2} and \ref{noiselessMain} respectively.

We will first consider the data consistency constraint in \eqref{l0probNoiseless} and determine possible feasible solutions. We observe that all the points in any specified cluster can share a centre without violating the data consistency constraint: 
\begin{noiselessResult}
\label{noiselessLemmA.0}
Consider any two points $\mathbf x_1 $ and $\mathbf x_2 $ from the same cluster. A solution $\mathbf u$ exists for the following equations:
	\begin{eqnarray}
\mathbf S_i\,\mathbf x_i &=& \mathbf S_i\, \mathbf u; ~~i=1,2
\end{eqnarray}
with probability $1$.
\end{noiselessResult}
The proof for the above result is trivial in this special case, since all points in the same cluster are the same. We now consider two points from different clusters. 

\begin{noiselessResult}
\label{noiselessLemmA.1}
Consider two points $\mathbf x_1 $ and $\mathbf x_2 $ from different clusters. A solution $\mathbf u$ exists for the following equations:
\begin{equation}
\label{common}
\mathbf S_i\,\mathbf x_i = \mathbf S_i\, \mathbf u; ~~i=1,2
\end{equation}
with low probability, when the sampling probability $p_0$ is high and coherence $\mu_0$ is low.	
\end{noiselessResult} 
By definition, $\mathbf S_1 = \mathbf S_{\mathcal I_1}$ and $\mathbf S_2 = \mathbf S_{\mathcal I_2}$, where $\mathcal I_1$ and $\mathcal I_2$ are the index sets of the features that are sampled (not missing) in $\mathbf x_1$ and $\mathbf x_2$ respectively. We observe that \eqref{common} can be satisfied, iff:
\begin{equation}
\label{commonsamples}
\mathbf S_{\mathcal I_1 \cap \mathcal I_2} (\mathbf x_1 - \mathbf x_2) = \mathbf 0
\end{equation}
which implies that the features of $\mathbf x_1$ and $\mathbf x_2$ are the same on the index set $\mathcal I_1 \cap \mathcal I_2$. If the probability of sampling $p_0$ is sufficiently high, then the number of samples at commonly observed locations:
\begin{equation}
|\mathcal I_1 \cap \mathcal I_2| = q
\end{equation}
will be high, with high probability. If the coherence $\mu_0$ defined in assumption A3 is low, then with high probability the vector $\mathbf x_1 - \mathbf x_2$ does not have $q$ entries that are equal to 0. In other words, the cluster memberships are not determined by only a few features. Thus, for a small value of $\mu_0$ and high $p_0$, we can ensure that \eqref{commonsamples} occurs with very low probability. We now generalize the above result to obtain the following:
\begin{noiselessResult}
\label{noiselessLemmA.2}
Assume that $\{\mathbf x_i: i\in \mathcal I, |\mathcal I|= M\}$ is a set of points chosen randomly from multiple clusters (not all are from the same cluster). A solution $\mathbf u$ exists for the following equations:
\begin{equation}
\mathbf S_i\,\mathbf x_i  = \mathbf S_i\, \mathbf u; ~\forall i \in \mathcal I
\end{equation}
with low probability, when the sampling probability $p_0$ is high and coherence $\mu_0$ is low.
\end{noiselessResult}
The key message of the above result is that large mis-classified clusters are highly unlikely. We will show that all feasible solutions containing small mis-classified clusters are associated with higher cost than the correct solution. Thus, we can conclude that the algorithm recovers the ground truth solution with high probability, as summarized by the following result.
\begin{noiselessResult}
\label{noiselessMain}
The optimization problem \eqref{l0probNoiseless} results in the ground-truth clustering with a high probability if the sampling probability $p_0$ is high and the coherence $\mu_0$ is low.
\end{noiselessResult}

\subsection{Noisy Clusters with Missing Entries}
\label{generalcase}
We will now consider the general case of noisy clusters with missing entries, and will determine the conditions required for \eqref{l0prob} to yield successful recovery of clusters. The reasoning behind the proof in the general case is similar to that for the special case discussed in the previous sub-section. Before proceeding to the statement of the lemmas and theorems, we define the following quantities:

\begin{itemize}
\item Upper bound for probability that two points have less than $\frac{p_0^2P}{2}$ commonly observed locations:
\begin{equation}
\label{gamma0}
\gamma_0 \coloneqq (\frac{e}{2})^{-\frac{p_0^2P}{2}}
\end{equation}

\item Given that two points from different clusters have more than $\frac{p_0^2P}{2}$ commonly observed locations, upper bound for probability that they can yield the same $\mathbf u$ without violating the constraints in \eqref{l0prob}:
\begin{equation}
\label{delta0}
\delta_0 \coloneqq e^{-\frac{p_0^2P(1-\kappa^2)^2}{\mu_0^2}}
\end{equation}

\item Upper bound for probability that two points from different clusters can yield the same $\mathbf u$ without violating the constraints in \eqref{l0prob}:
\begin{equation}
\label{beta0}
\beta_0 \coloneqq 1-(1-\delta_0)(1-\gamma_0)
\end{equation}

\item  Upper bound for failure probability of \eqref{l0prob}:
\begin{equation}
\label{eta0}
\eta_0 \coloneqq \sum_{\{m_j\} \in \mathcal S}\left[\beta_0^{\frac{1}{2}(M^2-\sum_j {m_j^2})} \prod_j {M \choose m_j}\right]
\end{equation}
where $\mathcal S$ is the set of all sets of positive integers $\{m_j\}$ such that: $2 \leq \mathcal U(\{m_j\}) \leq K$ and $\sum_j m_j = M$. Here, the function $\mathcal U$ counts the number of non-zero elements in a set. For example, if $K=2$ then $\mathcal S$ contains all sets of $2$ positive integers $\{m_1, m_2\}$, such that $m_1 + m_2 = M$. Thus, $\mathcal S = \{\{1,M-1\}, \{2,M-2\}, \{3,M-3\}, \ldots, \{M-1,1\}\}$ and \eqref{eta0} reduces to:
\begin{equation}
\eta_0 =  \sum_{i=1}^{M-1}\left[\beta_0^{i(M-i)} {M \choose i}^2\right]
\end{equation}
\item We note that the expression for $\eta_0$ is quite involved. Hence, to provide some intuition, we simplify this expression for the special case where there are only two clusters. 
Under the assumption that $\log \beta_0 \leq \frac{1}{M-1} + \frac{2}{M-2}\log \frac{1}{M-1}$, it can be shown that $\eta_0$ is upper-bounded as:
\begin{equation}
\label{etaApprox}
\begin{split}
\eta_0 & = \sum_{i=1}^{M-1}\left[\beta_0^{i(M-i)} {M \choose i}^2\right]\\ &\leq M^3 \beta_0^{M-1} \\ &\coloneqq \eta_{0,{\rm approx}}
\end{split}
\end{equation}
The above upper bound is derived in Appendix \ref{appUB}.

\end{itemize}
We now state the results for clustering with missing entries in the general noisy case. The following two lemmas are generalizations of Results \ref{noiselessLemmA.0} and \ref{noiselessLemmA.1} to the noisy case. 

\begin{lemma}
\label{cent}
Consider any two points $\mathbf x_1 $ and $\mathbf x_2 $ from the same cluster. A solution $\mathbf u$ exists for the following equations:
	\begin{eqnarray}
\label{samecluster}
\|\mathbf S_i\,(\mathbf x_i-\mathbf u)\|_{\infty} &\leq& {\frac{\epsilon}{2}}; ~ ~i=1,2
\end{eqnarray}
with probability $1$.
\end{lemma}
The proof of this lemma is in Appendix \ref{appCent}.
\begin{lemma}
\label{partDist}
Consider any two points $\mathbf x_1$ and $\mathbf x_2$ from different clusters, and assume that $\kappa<1$. A solution $\mathbf u$ exists for the following equations:
	\begin{eqnarray}
\|\mathbf S_i\,(\mathbf x_i-\mathbf u)\|_{\infty} &\leq& {\frac{\epsilon}{2}}; ~ ~i=1,2
\end{eqnarray}
with probability less than $\beta_0$.
\end{lemma}
The proof of this lemma is in Appendix \ref{lemma2.2proof}. We note that $\beta_0$ decreases with a decrease in $\kappa$. A small $\epsilon$ implies less variability within clusters and a large $\delta$ implies well-separated clusters, together resulting in a low value of $\kappa$. Both these characteristics are desirable for clustering and result in a low value of $\beta_0$. This lemma also demonstrates that the coherence assumption is important in ensuring that the sampled entries are sufficient to distinguish between a pair of points from different clusters. As a result, $\beta_0$ decreases with a decrease in the value of $\mu_0$. As expected, we also observe that $\beta_0$ decreases with an increase in $p_0$. 

The above result can be generalized to consider a large number of points from multiple clusters. If we choose $M$ points such that not all of them belong to the same cluster, then it can be shown that with high probability, they cannot share the same $\mathbf u$ without violating the constraints in \eqref{l0prob}. This idea (a generalization of Result \ref{noiselessLemmA.2}) is expressed in the following lemma:

\begin{lemma}
\label{smallClusters}
Assume that $\{\mathbf x_i: i\in \mathcal I, |\mathcal I|= M\}$ is a set of points chosen randomly from multiple clusters (not all are from the same cluster). If  $\kappa<1$, a solution $\mathbf u$ does not exist for the following equations:
\begin{equation}
\|\mathbf S_i\,(\mathbf x_i-\mathbf u)\|_{\infty} \leq {\frac{\epsilon}{2}}; ~ ~\forall i \in \mathcal I
\end{equation} 
with probability exceeding $1 - \eta_0$.
\end{lemma}

The proof of this lemma is in Appendix \ref{appSC}. We note here, that for a low value of $\beta_0$ and a high value of $M$ (number of points in each cluster), we will arrive at a very low value of $\eta_0$. Using Lemmas \ref{cent}, \ref{partDist} and \ref{smallClusters}, we now move on to our main result which is a generalization of Result \ref{noiselessMain}: 

\begin{theorem}
\label{mainResult}
If $\kappa<1$, the solution to the optimization problem \eqref{l0prob} is identical to the ground-truth clustering with probability exceeding $1 - \eta_0$.
\end{theorem}

The proof of the above theorem is in Appendix \ref{appMR}. The reasoning follows from Lemma \ref{smallClusters}. It is shown in the proof that all solutions with cluster sizes smaller than $M$ are associated with a higher cost than the ground-truth solution. 

\subsection{Clusters without Missing Entries}
\label{nomissing}
We now study the case where there are no missing entries. In this special case, optimization problem \eqref{l0prob} reduces to:
\begin{equation}
\label{l0probFull}
\begin{split}
\{\mathbf u_i^{*}\} = & \min_{\{\mathbf u_i\}} \sum_{i=1}^{KM}\sum_{j=1}^{KM}\|\mathbf u_i - \mathbf u_j\|_{2,0}\\ & \mbox{ s.t } \|\mathbf x_i - \mathbf u_i\|_{\infty} \leq \frac{\epsilon}{2}, ~i\in\{1 \ldots KM\}
\end{split}
\end{equation}
We have the following theorem guaranteeing successful recovery for clusters without missing entries:

\begin{theorem}
	\label{noMissingFinal}
	If $\kappa<1$, the solution to the optimization problem \eqref{l0probFull} is identical to the ground-truth clustering.
\end{theorem}

The proof for the above Theorem is in Appendix \ref{appNM}. We note that the above result does not consider any particular distribution of the points in each cluster. Instead, if we consider that the points in each cluster are sampled from certain particular probability distributions such as the uniform random distribution, then a larger $\kappa$ is sufficient to ensure success with high probability. In the general case where no such distribution is assumed, we cannot make a probabilistic argument, and a smaller $\kappa$ is required. We now consider a special case, where the noise $\mathbf n_k(m)$ is a zero mean uniform random variable $\sim U(-\epsilon/2,\epsilon/2)$. Thus, the points within each cluster are uniformly distributed in a cube of side $\epsilon$. We note that $\delta$ is now a random variable, and thus instead of using the constant $\kappa = \frac{\epsilon \sqrt{P}}{\delta}$ (as in previous lemmas), we define the following constant:
\begin{equation}
\label{kappaDash}
\kappa'=\frac{\epsilon \sqrt{P}}{c} 
\end{equation}
where $c$ is defined as the minimum separation between the centres of any $2$ clusters in the dataset:
\begin{equation}
\label{cDef}
\min_{\{k,l\}}\|\mathbf c_{k} -\mathbf c_{l}\|_{2} \geq c; ~\forall\; k\neq l
\end{equation}
We also define the following quantity:
\begin{equation}
\label{beta1}
\beta_1=e^{-\frac{P(1-\frac{5}{6}\kappa'^{2})^2}{8 \kappa'^{2}}}
\end{equation}
We arrive at the following result for two points in different clusters:
\begin{lemma}
	\label{noMissingUniform}
	Let $\kappa' < \sqrt{\frac{6}{5}}$. If the points in each cluster follow a uniform random distribution, then for two points $\mathbf x_1$ and $\mathbf x_2$ belonging to different clusters, a solution $\mathbf u$ exists for the following equations:
	\begin{eqnarray}
\label{dataconsistencynoise}
\|\mathbf x_i-\mathbf u\|_{\infty} &\leq& {\frac{\epsilon}{2}}; ~ ~i=1,2
\end{eqnarray}
with probability less than $\beta_1$.
\end{lemma}
The proof for the above lemma is in Appendix \ref{appNMU}. This implies that for $\kappa' <\sqrt{\frac{6}{5}}$, two points from different clusters cannot be misclassified to a single cluster with high probability. As $\eta_0$ is expressed in terms of $\beta_0$ in \eqref{eta0}, we can also express $\eta_1$ in terms of $\beta_1$. We get the following guarantee for perfect clustering:
\begin{theorem}
	\label{noMissingUnifFinal}
	If the points in each cluster follow a uniform random distribution and $\kappa' < \sqrt{\frac{6}{5}}$ , then  the solution to the optimization problem \eqref{l0probFull} is identical to the ground-truth clustering with probability exceeding $1 - \eta_1$.
\end{theorem}
Note that $\kappa = \kappa'\frac{c}{\delta}$. Thus, the above result allows for values $\kappa>1$. Our results show that if we do not consider the distribution of the points, then we arrive at the bound  $\kappa < 1$ with and without missing entries, as seen from Theorems \ref{mainResult} and \ref{noMissingFinal} respectively. A uniform random distribution can also be assumed in the case of missing entries. Similar to Theorem \ref{noMissingUnifFinal}, we expect an improved bound for the case with missing entries as well. 

\section{Relaxation of the $\ell_0$ penalty}

\subsection{Constrained formulation}

We propose to solve a relaxation of the optimization problem \eqref{l0prob}, which is more computationally feasible. The relaxed problem is given by:
\begin{equation}
\label{relax1}
\begin{split}
	\{\mathbf u_i^{*}\} = & \min_{\{\mathbf u_i\}} \sum_{i=1}^{KM}\sum_{j=1}^{KM}\phi\left(\|\mathbf u_i - \mathbf u_j\|_2\right)\\ & \mbox{ s.t } \|\mathbf S_i(\mathbf x_i - \mathbf u_i)\|_\infty \leq {\frac{\epsilon}{2}}, i\in\{1 \ldots KM\}
\end{split}
\end{equation}
where $\phi$ is a function approximating the $\ell_{0}$ norm. Some examples of such functions are:
\begin{itemize}
	\item $\ell_p$ norm: $\phi(x) = |x|^p$, for some $0<p<1$.
	\item $H_1$ penalty: $\phi(x) = 1-e^{-\frac{x^2}{2\sigma^2}}$.
\end{itemize}
These functions approximate the $\ell_0$ penalty more accurately for lower values of $p$ and $\sigma$, as illustrated in Fig \ref{FigApproxFunc}. We reformulate the problem using a majorize-minimize strategy. Specifically, by majorizing the penalty $\phi$ using a quadratic surrogate functional, we obtain:

\begin{figure}[!t]
	\centering
	\center{\includegraphics[width=0.49\textwidth]{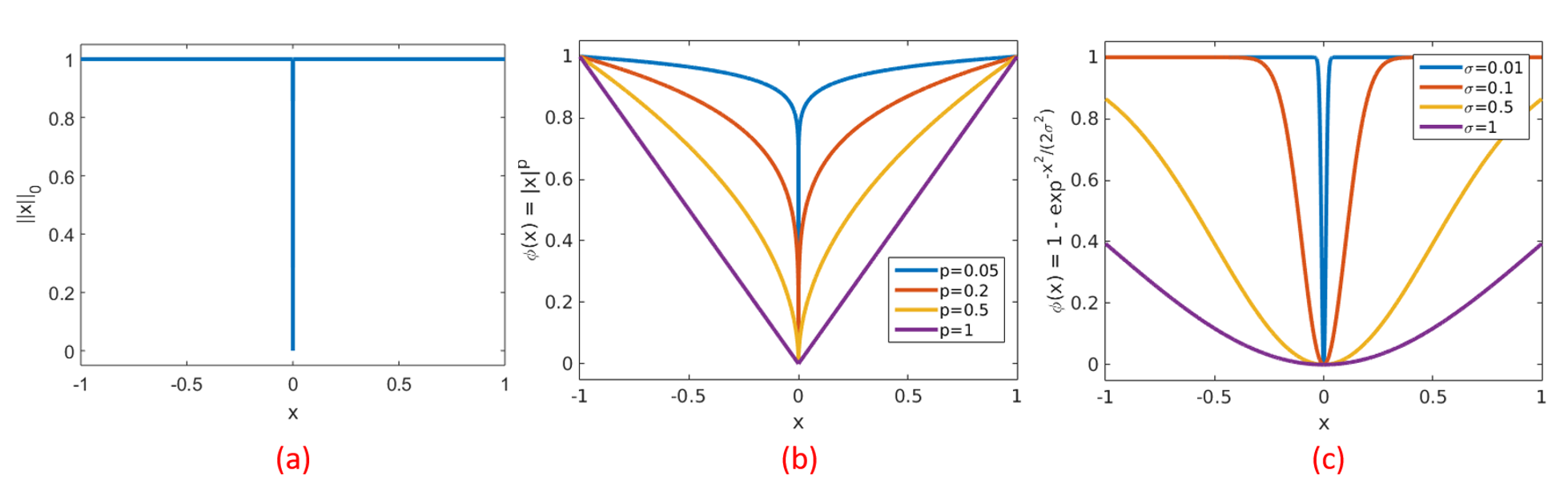}}
	\caption{\small Different penalty functions $\phi$. (a) The $\ell_0$ norm (b) The $\ell_p$ penalty function which is non-convex for $0<p<1$ and convex for $p=1$ (c) The $H_1$ penalty function. The $\ell_p$ and $H_1$ penalties closely approximate the $\ell_0$ norm for low values of $p$ and $\sigma$ respectively.}
	\label{FigApproxFunc}
\end{figure}

\begin{equation}
\phi(x) \leq w(x) x^2 + d
\end{equation}
where $w(x) = \frac{\phi^{'}(x)}{2x}$, and $d$ is a constant. For the two penalties considered here, we obtain the weights as:
\begin{itemize}
	\item $\ell_p$ norm: $w(x) = (\frac{2}{p}x^{(2-p)}+\alpha)^{-1} $ . The infinitesimally small $\alpha$ term is introduced to deal with situations where $x=0$. For non-zero $x$, we get the expression $w(x) \approx \frac{p}{2}x^{p-2}$.
	\item $H_1$ penalty: $w(x) = \frac{1}{2\sigma^2}e^{-\frac{x^2}{2\sigma^2}}$.
\end{itemize}
We can now state the majorize-minimize formulation for problem \eqref{relax1} as:
\begin{equation}
\begin{split}
\label{relaxConstrained}
\{\mathbf u_i^*, w_{ij}^*\} = &\arg \min_{\{\mathbf u_i,w_{ij}\}} \sum_{i=1}^{KM} \sum_{j=1}^{KM} w_{ij}~\|\mathbf u_i - \mathbf u_j\|_2^2\\ &\mbox{ s.t } \|\mathbf S_i(\mathbf x_i - \mathbf u_i)\|_\infty \leq \;{\frac{\epsilon}{2}}, i\in\{1 \ldots KM\}
\end{split}
\end{equation}
where the constant $d$ has been ignored. In order to solve problem \eqref{relaxConstrained}, we alternate between two sub-problems till convergence. At the $n^{th}$ iteration, these sub-problems are given by:
\begin{equation}
\label{sub1IRLS}
w_{ij}^{(n)} = \frac{\phi^{'}\left(\|\mathbf u_i^{(n-1)} - \mathbf u_j^{(n-1)}\|_2\right)}{2\|\mathbf u_i^{(n-1)} - \mathbf u_j^{(n-1)}\|_2}
\end{equation}

\begin{equation}
\label{sub2IRLS}
\begin{split}
\{\mathbf u_i^{(n)}\} = & \arg \min_{\{\mathbf u_i\}} \sum_{i=1}^{KM} \sum_{j=1}^{KM} w_{ij}^{(n)}\|\mathbf u_i - \mathbf u_j\|_2^2\\ &\mbox{ s.t } \|\mathbf S_i(\mathbf x_i - \mathbf u_i)\|_\infty \leq \;{\frac{\epsilon}{2}}, i\in\{1 \ldots KM\}
\end{split}
\end{equation}

\begin{figure}[!t]
	\centering
	\center{\includegraphics[width=0.49\textwidth]{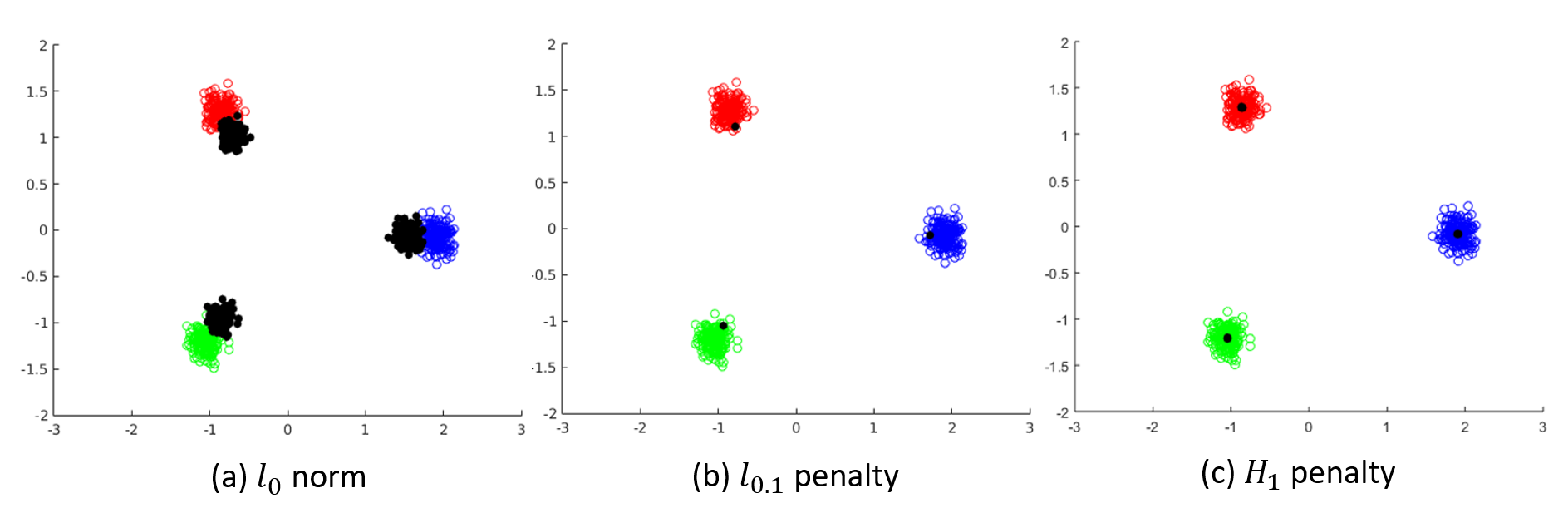}}
	\caption{Comparison of different penalties. We show here the 2 most significant principal components of the solutions obtained using the IRLS algorithm. (a) It can be seen that the $\ell_{1}$ penalty is unable to cluster the points even though the clusters are well-separated. (b) The $\ell_{0.1}$ penalty is able to cluster the points correctly. However, the cluster-centres are not correctly estimated. (c) The $H_1$ penalty correctly clusters the points and also gives a good estimate of the centres.}
	\label{compareFig}
\end{figure}

\subsection{Unconstrained formulation}

For larger datasets, it might be computationally intensive to solve the constrained problem. In this case, we propose to solve the following unconstrained problem:

\begin{equation}
\label{relax}
\{\mathbf u_i^*\} = \arg \min_{\{\mathbf u_i\}} \sum_{i=1}^{KM}\|\mathbf S_i(\mathbf u_i - \mathbf x_i)\|_2^2 + \lambda \sum_{i=1}^{KM} \sum_{j=1}^{KM} \phi(\|\mathbf u_i - \mathbf u_j\|_2)
\end{equation}
As before, we can state the majorize-minimize formulation for problem \eqref{relax} as:
\begin{equation}
\label{relaxIRLS}
\begin{split}
 \{\mathbf u_i^*, w_{ij}^*\} = \arg \min_{\{\mathbf u_i,w_{ij}\}}\sum_{i=1}^{KM}\| &\mathbf S_i(\mathbf u_i - \mathbf x_i)\|_2^2\\ + & \lambda \sum_{i=1}^{KM}\sum_{j=1}^{KM} w_{ij}\|\mathbf u_i - \mathbf u_j\|_2^2
 \end{split}
\end{equation}
In order to solve the problem \eqref{relaxIRLS}, we alternate between two sub-problems till convergence. The $1^{st}$ sub-problem is the same as \eqref{sub1IRLS}. The $2^{nd}$ sub-problem is given by:
\begin{equation}
\label{sub2IRLSunc}
\begin{split}
\{\mathbf u_i^{(n)}\} = \arg \min_{\{\mathbf u_i\}} \sum_{i=1}^{KM}\| &\mathbf S_i(\mathbf u_i - \mathbf x_i)\|_2^2\\ + &\lambda \sum_{i=1}^{KM} \sum_{j=1}^{KM} w_{ij}^{(n)}\|\mathbf u_i - \mathbf u_j\|_2^2
\end{split}
\end{equation}

\subsection{Comparison of penalties}
\label{compPen}

We compare the performance of different penalties when used as a surrogate for the $\ell_0$ norm. For this purpose, we use a simulated dataset with points in $\mathbb R^{50}$ belonging to $3$ well-separated clusters, with $200$ points in each cluster. For this particular experiment, we considered $\mathbf x_1, \mathbf x_2,\ldots, \mathbf x_{200} \in \mathcal C_1$, $\mathbf x_{201}, \mathbf x_{202},\ldots, \mathbf x_{400} \in \mathcal C_2$ and $\mathbf x_{401}, \mathbf x_{402},\ldots, \mathbf x_{600} \in \mathcal C_3$. We do not consider the presence of missing entries for this experiment. We solve problem \eqref{relax} to cluster the points using the $\ell_{1}$, $\ell_{p}$ (for $p=0.1$) and $H_1$ (for $\sigma=0.5$) penalties. The results are shown in Fig \ref{compareFig}. Only for the purpose of visualization, we take a PCA of the data matrix $\mathbf X \in \mathbb R^{50\times 600}$ and retain the $2$ most significant principal components to get a matrix of points $\in \mathbb R^{2\times 600}$. These points are plotted in the figure, with red, blue and green representing points from different clusters. We similarly obtain the $2$ most significant components of the estimated centres and plot the resulting points in black. In (b) and (c), we note that $\mathbf u_1^* = \mathbf u_2^* = \ldots = \mathbf u_{200}^*$, $\mathbf u_{201}^*= \mathbf u_{202}^*=\ldots = \mathbf u_{400}^*$ and $\mathbf u_{401}^*=\mathbf u_{402}^*=\ldots= \mathbf u_{600}^*$. Thus, the $\ell_p$ penalty and the $H_1$ penalty are able to correctly cluster the points. This behaviour is not seen in (a). Thus it is concluded that the convex $\ell_{1}$ penalty is unable to cluster the points. 

The cluster-centres estimated using the $\ell_{p}$ penalty are inaccurate. The $H_1$ penalty out-performs the other two penalties and accurately estimates the cluster-centres. We can explain this behaviour intuitively by observing the plots in Fig \ref{FigApproxFunc}. The $\ell_{1}$ norm penalizes differences between all pairs of points. The  $\ell_{0.1}$ semi-norm penalizes differences between points that are close. Due to the saturating nature of the penalty, it does not heavily penalize differences between points that are further away. The same is true for the $H_1$ penalty. However, we note that the $H_1$ penalty saturates to $1$ very quickly, similar to the $\ell_0$ norm. This behaviour is missing for the $\ell_{0.1}$ penalty. For this reason, it is seen that the $\ell_{0.1}$ penalty also penalizes inter-cluster distances (unlike the $H_1$ penalty), and shrinks the distance between the estimated centres of different clusters. 

\begin{figure*}
	\centering
	\center{\includegraphics[width=1\textwidth]{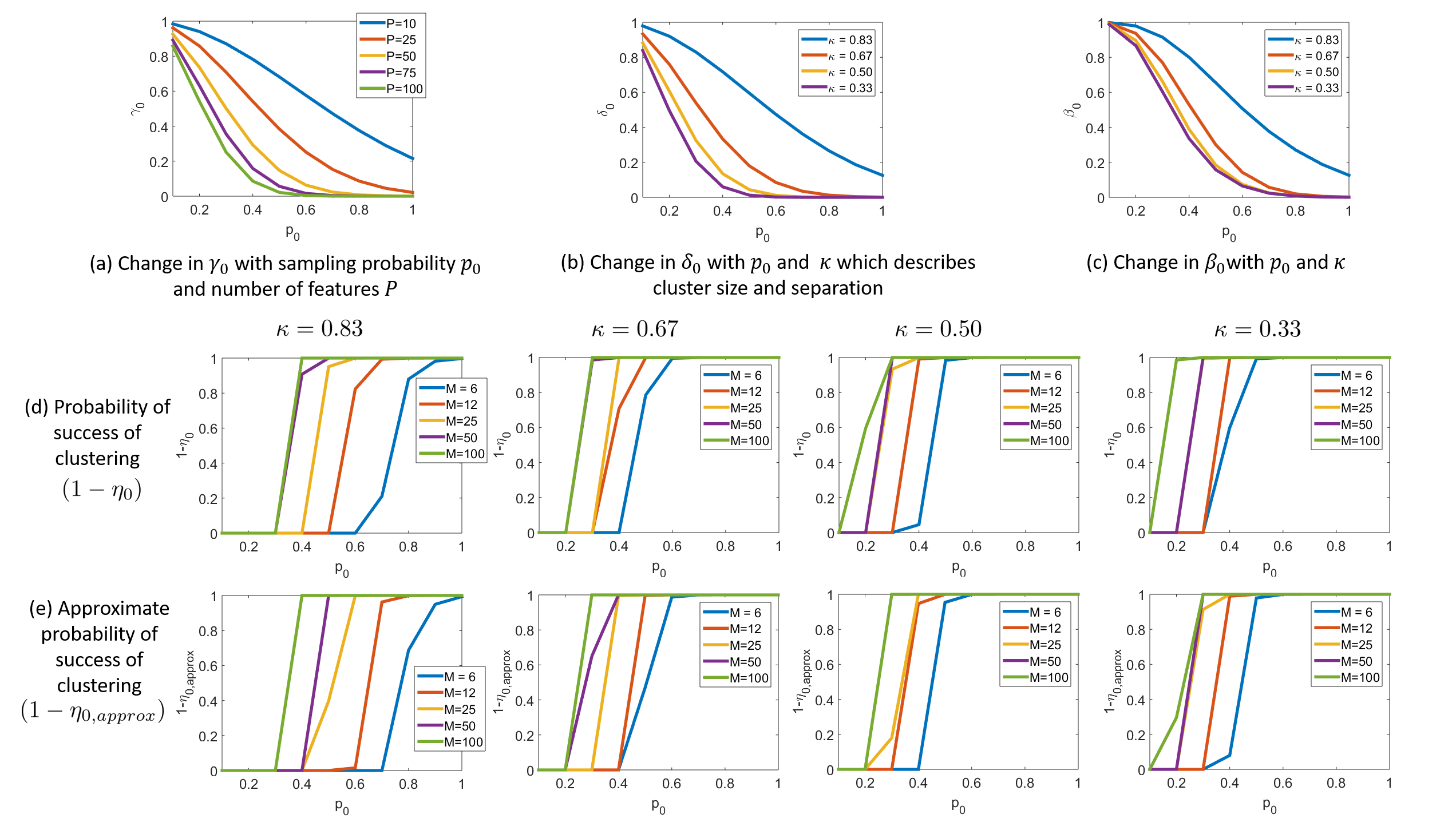}}
	\caption{Study of Theoretical Guarantees. The quantities $\gamma_0, \delta_0$ and $\beta_0$ defined in Section \ref{generalcase} are studied in (a), (b) and (c) respectively. In (b) and (c), $P=50$ and $\mu_0=1.5$ are assumed. $\beta_0$ gives the probability that 2 points from different clusters can share a centre. As expected, this value decreases with increase in $p_0$ and decrease in $\kappa$. Considering $K=2$ clusters, a lower bound for the  probability of successful clustering $(1-\eta_0)$ using the proposed algorithm is shown in (d) for different values of $\kappa$. The approximate values $(1-\eta_{0,{\rm approx}})$ computed using  \eqref{etaApprox} are shown in (e).}
	\label{FigTh}
\end{figure*}

\begin{figure}
	\centering
	\center{\includegraphics[width=0.49\textwidth]{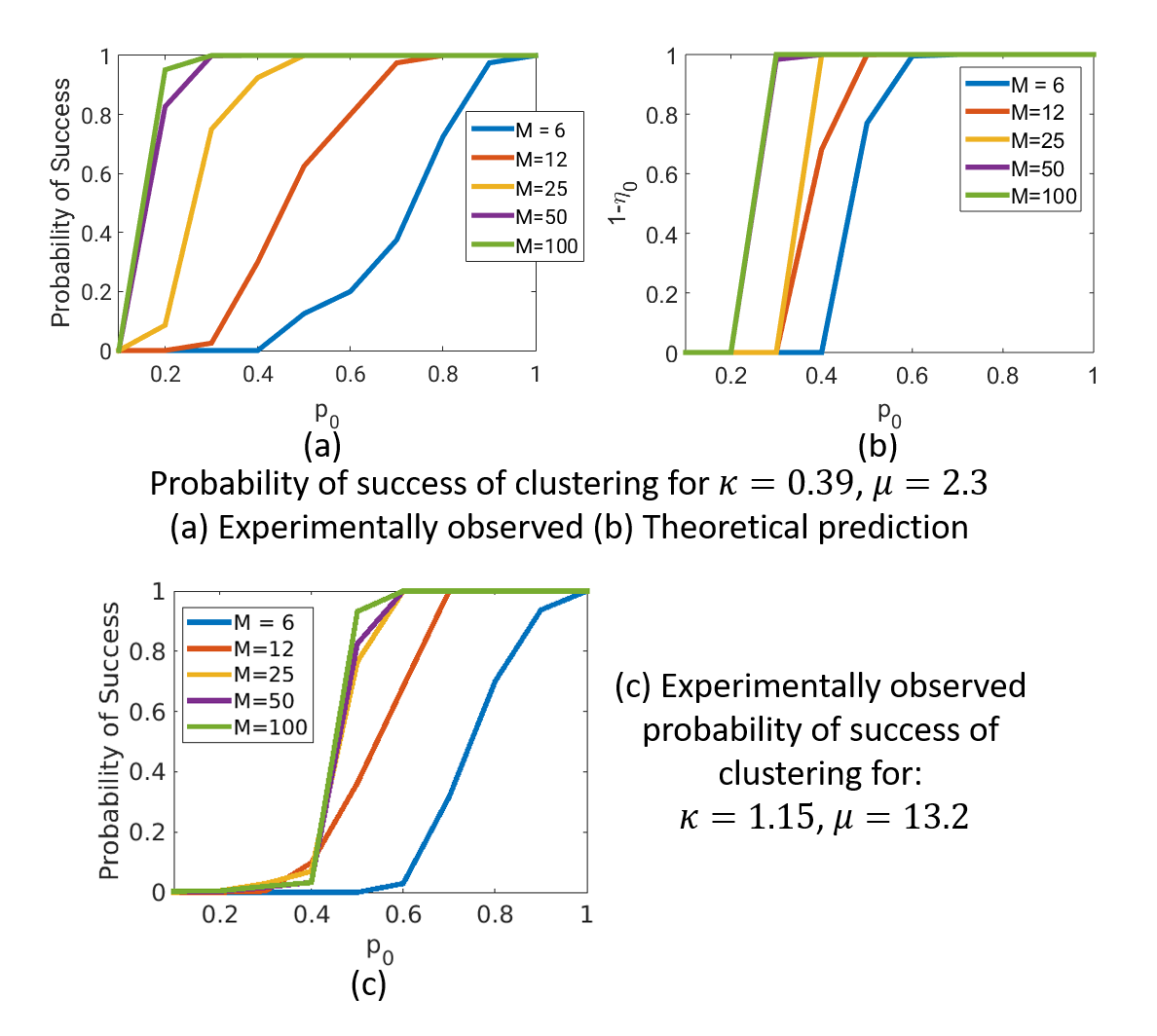}}
	\caption{Experimental results for probability of success. Guarantees are shown for a simulated dataset with $K=2$ clusters. The clustering was performed using \eqref{relaxConstrained} with an $H_1$ penalty and partial distance based initialization. For (a) and (b) it is assumed that $\kappa = 0.39$ and $\mu_0 = 2.3$. (a) shows the experimentally obtained probability of success of clustering for clusters with points from a uniform random distribution. (b) shows the theoretical lower bound for the probability of success. (c) shows the experimentally obtained probability of success for a more challenging dataset with $\kappa = 1.15$ and $\mu_0 = 13.2$. Note that we do not have theoretical guarantees for this case, since our analysis assumes that $\kappa < 1$.}
	\label{FigExpPlots}
\end{figure}

\subsection{Initialization Strategies}

Our experiments emphasize the need for a good initialization of the weights $w_{ij}$ for convergence to the correct cluster centre estimates. This dependence on the initial value arises from the non-convexity of the optimization problem. We consider two different strategies for initializing the weights:
\begin{itemize}
	\item Partial Distances: Consider a pair of points $\mathbf x_1, \mathbf x_2$ observed by sampling matrices $\mathbf S_1 = \mathbf S_{\mathcal I_1}$ and $\mathbf S_2 = \mathbf S_{\mathcal I_2}$ respectively. Let the set of common indices be $\omega \coloneqq \mathcal I_1 \cap \mathcal I_2$. We define the partial distance as $\|\mathbf y_\omega\| = \sqrt{\frac{P}{|\omega|}}\|\mathbf x_{1\omega} - \mathbf x_{2\omega}\|$, where $\mathbf x_{i\omega}$ represents the set of entries of $\mathbf x_i$ restricted to the index set $\omega$. Instead of the actual distances which are not available, the partial distances $\|\mathbf y_\omega\|$ can be used for computing the weights.
	\\
	\item Imputation Methods: The weights can be computed from estimates $\{\mathbf u_i^{(0)}\}$, where:
	\begin{equation}
	\mathbf u_i^{(0)} = \mathbf S_i \mathbf x_i + (\mathbf I- \mathbf S_i) \mathbf m 
	\end{equation}
	Here $\mathbf m$ is a constant vector, specific to the imputation technique. The zero-filling technique corresponds to $\mathbf m = \mathbf 0$. Better estimation techniques can be derived where the $j^{th}$ row of $\mathbf m$ can be set to the mean of all measured values in the $j^{th}$ row of $\mathbf X$.
\end{itemize}
We will observe experimentally that for a good approximation of the initial weights $\mathbf W^{(0)}$, we get the correct clustering. Conversely, the clustering fails for a bad initial guess. Our experiments demonstrate the superiority of a partial distance based initialization strategy over a zero-filled initialization.

\section{Results}

We study the proposed theoretical guarantees for Theorem \ref{mainResult} for different settings. We also test the proposed algorithm on simulated and real datasets. The simulations are used to study the performance of the algorithm with change in parameters such as fraction of missing entries, number of points to be clustered etc. We also study the effect of different initialization techniques on the algorithm performance. We demonstrate the algorithm on  the publicly available Wine dataset \cite{Lichman:2013}, and use the algorithm to reconstruct a dataset of under-sampled cardiac MR images.

\subsection{Study of Theoretical Guarantees}

We observe the behaviour of the quantities $\gamma_0, \delta_0, \beta_0, \eta_0$ and $\eta_{0,{\rm approx}}$ (defined in section \ref{generalcase}) as a function of parameters $p_0, P, \kappa$ and $M$. Fig \ref{FigTh} shows a few plots that illustrate the change in these quantities as the different parameters are varied. $\gamma_0$ is an upper bound for the probability that a pair of points have $< \frac{p_0^2P}{2}$ entries observed at common locations. In Fig \ref{FigTh} (a), the change in $\gamma_0$ is shown as a function of $p_0$ for different values of $P$. In subsequent plots, we fix $P=50$ and $\mu_0=1.5$. $\delta_0$ is an upper bound for the probability that a pair of points from  different clusters can share a common centre, given that $\geq \frac{p_0^2P}{2}$ entries are observed at common locations. In Fig \ref{FigTh} (b), the change in $\delta_0$ is shown as a function of $p_0$ for different values of $\kappa$. In Fig \ref{FigTh} (c), the behaviour of $\beta_0 = 1-(1-\gamma_0)(1- \delta_0)$ is shown, which is the probability mentioned in Lemma \ref{partDist}.

We consider the two cluster setting, (i.e. $K=2$) for subsequent plots. $\eta_0$ is the probability of failure of the clustering algorithm \eqref{l0prob}. In (d) and (e), plots are shown for $(1-\eta_0)$ and $(1-\eta_{0,{\rm approx}})$ as a function of $p_0$ for different values of $\kappa$ and $M$. Here, $\eta_{0,{\rm approx}}$ is an upper bound for $\eta_0$ computed using \eqref{etaApprox}. As expected, the probability of success of the clustering algorithm increases with  increase in $p_0$ and $M$ and decrease in $\kappa$.

\subsection{Clustering of Simulated Data}

We simulated datasets with $K=2$ disjoint clusters in $\mathbb R^{50}$ with a varying number of points per cluster ($M=6,12,25,50,100$). The points in each cluster follow a uniform random distribution. We study the probability of success of the $H_1$ penalty based constrained clustering algorithm (with partial-distance based initialization) as a function of $\kappa$, $M$ and $p_0$. For a particular set of parameters the experiment was conducted $20$ times to compute the probability of success of the algorithm. Between these $20$ trials, the cluster-centers remain the same, while the points sampled from these clusters are different and the locations of the missing entries are different. Fig \ref{FigExpPlots} (a) shows the result for datasets with $\kappa = 0.39$ and $\mu_0 = 2.3$. The theoretical guarantees for successfully clustering the dataset are shown in (b). Note that the theoretical guarantees do not assume that the points are taken from a uniform random distribution. Also, the theoretical bounds assume that we are solving the original problem using a $\ell_0$ norm, whereas the experimental results were generated for the $H_1$ penalty. Our theoretical guarantees hold for $\kappa < 1$. However, we demonstrate in (c) that even for the more challenging case where $\kappa = 1.15$ and $\mu_0 = 13.2$, our clustering algorithm is successful. Note that we do not have theoretical guarantees for this case. However, by assuming a uniform random distribution on the points, we expect that we can get better theoretical guarantees (similar to Theorem \ref{noMissingUnifFinal} for the case without missing entries).

Clustering results with $K=3$ simulated clusters are shown in Fig \ref{FigSim1}. We simulated Dataset-1 with $K=3$ disjoint clusters in $\mathbb R^{50}$ and $M=200$ points in each cluster. In order to generate this dataset, $3$ cluster centres in $\mathbb R^{50}$ were chosen from a uniform random distribution. The distances between the $3$ pairs of cluster-centres are $3.5$, $2.8$ and $3.3$ units respectively. For each of these $3$ cluster centres, $200$ noisy instances were generated by adding zero-mean white Gaussian noise of variance 0.1. The dataset was sub-sampled with varying fractions of missing entries ($p_0=1,0.9,0.8,\ldots,0.3,0.2$). The locations of the missing entries were chosen uniformly at random from the full data matrix. We also generate Dataset-2 by  halving the distance between the cluster centres, while keeping the intra-cluster variance fixed. We test both the constrained \eqref{relax1} and unconstrained \eqref{relax} formulations of our algorithm on these datasets. Both the proposed initialization techniques for the IRLS algorithm (i.e. zero-filling and partial-distance) are also tested here. Since the points lie in $\mathbb R^{50}$, we take a PCA of the points and their estimated centres (similar to Fig \ref{compareFig}) and plot the $2$ most significant components. The $3$ colours distinguish the points according to their ground-truth clusters. Each point $\mathbf x_i$ is joined to its centre estimate $\mathbf u_i^*$ by a line. As expected, we observe that the clustering algorithms are more stable with fewer missing entries. We also note that the results are quite sensitive to the initialization technique. We observe that the partial distance based initialization technique out-performs the zero-filled initialization. The unconstrained algorithm with partial distance-based initialization shows superior performance to the alternative schemes. Thus, we use this scheme for subsequent experiments on real datasets.

\subsection{Clustering of Wine Dataset}

We apply the clustering algorithm to the Wine dataset \cite{Lichman:2013}. The data consists of the results of a chemical analysis of wines from $3$ different cultivars. Each data point has $P=13$ features. The $3$ clusters have $59$, $71$ and $48$ points respectively, resulting in a total of $178$ data points. We created a dataset without outliers by retaining only $M=40$ points per cluster, resulting in a total of $120$ data points. We under-sampled these datasets using uniform random sampling with different fractions of missing entries. The results are displayed in Fig \ref{FigWine1} using the PCA technique as explained in the previous sub-section. It is seen that the clustering is quite stable and degrades gradually with increasing fractions of missing entries.

\subsection{Cardiac MR Image Reconstruction}

We apply the proposed algorithm to the reconstruction of a cardiac MR image time series. In MRI, samples are collected in the Fourier domain. However, due to the slow nature of the acquisition, only a small fraction of the Fourier samples can be acquired in each time frame. The goal of image reconstruction is to recover the image series from the incomplete Fourier observations. In the case of cardiac MRI, the different images in the time series appear in clusters determined  by the cardiac and respiratory phase. Thus, the proposed algorithm can be applied to the image reconstruction problem.

The cardiac data was acquired on a Siemens Aera MRI scanner at the University of Iowa. The subject was asked to breathe freely, and 10 radial lines of Fourier data was acquired to reconstruct each image frame. Fourier data corresponding to 1000 frames was acquired and the image series was reconstructed using the proposed unconstrained algorithm. We performed spectral clustering \cite{ng2001spectral} on the reconstructed images to form 20 clusters. A few reconstructed frames belonging to 2 different clusters are illustrated in Fig \ref{cardiacFig}. The images displayed have minimal artefacts and are of diagnostic quality.

\begin{figure*}
	\centering
	\center{\includegraphics[width=0.7\textwidth]{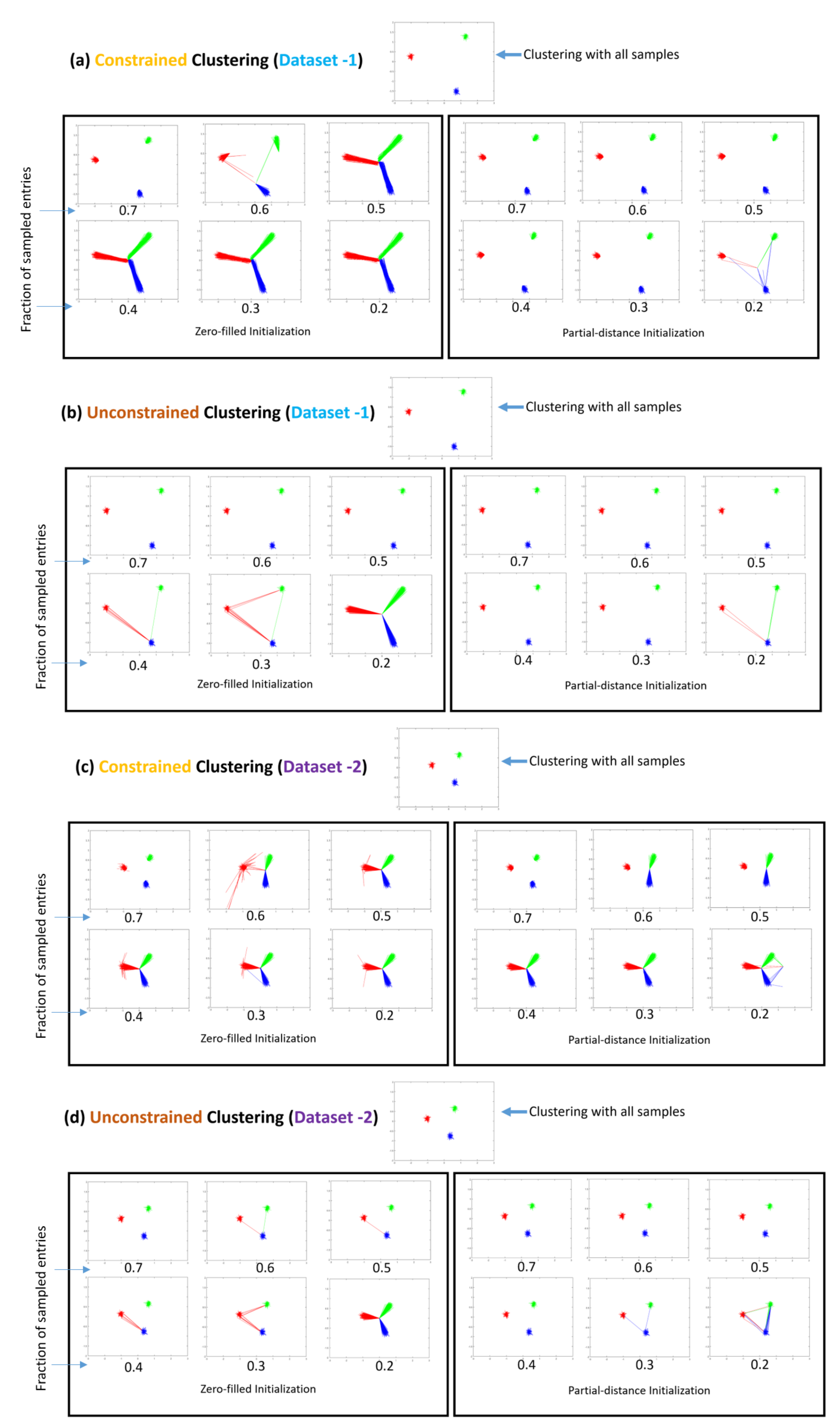}}
	\caption{Clustering results in simulated datasets. The $H_1$ penalty is used to cluster two datasets with varying fractions of missing entries. Both the constrained and unconstrained formulation results are presented with different initialization techniques (zero-filled and partial-distance based). We show here the 2 most significant principal components of the solutions. The original points $\{\mathbf x_i\}$ are connected to their cluster centre estimates $\{\mathbf u_i\}$ by lines. Inter-cluster distances in Dataset 2 are half of those in Dataset 1, while intra-cluster distances remain the same. Consequently, Dataset 1 performs better at a higher fraction of missing entries. For the unconstrained clustering formulation with partial-distance based initialization, the cluster centre estimates are relatively stable with varying fractions of missing entries.}
	\label{FigSim1}
\end{figure*}

\section{Discussion}

We have proposed a technique to cluster points when some of the feature values of all the points are unknown. We theoretically studied the performance of an algorithm that minimizes an $\ell_0$ fusion penalty subject to certain constraints relating to consistency with the known features. We concluded that under favourable clustering conditions, such as well-separated clusters with low intra-cluster variance, the proposed method performs the correct clustering even in the presence of missing entries. However, since the problem is NP-hard, we propose to use other penalties that approximate the $\ell_0$ norm. We observe experimentally that the $H_1$ penalty is a good surrogate for the $\ell_0$ norm. This non-convex saturating penalty is shown to perform better in the clustering task than previously used convex norms and penalties. We describe an IRLS based strategy to solve the relaxed problem using the surrogate penalty.

Our theoretical analysis reveals the various factors that determine whether the points will be clustered correctly in the presence of missing entries. It is obvious that the performance degrades with the decrease in the fraction of sampled entries ($p_0$). Moreover, it is shown that the difference between points from different clusters should have low coherence ($\mu_0$). This means that the expected clustering should not be dependent on only a few features of the points. Intuitively, if the points in different clusters can be distinguished by only $1$ or $2$ features, then a point missing these particular feature values cannot be clustered correctly. Moreover, we note that a high number of points per cluster ($M$), high number of features ($P$) and a low number of clusters ($K$) make the data less sensitive to missing entries. Finally, well-separated clusters with low intra-cluster variance (resulting in low values of $\kappa$) are desirable for correct clustering.

Our experimental results show great promise for the proposed technique. In particular, for the simulated data, we note that the cluster-centre estimates degrade gradually with increase in the fraction of missing entries. Depending on the characteristics of the data such as number of points and cluster separation distance, the clustering algorithm fails at some particular fraction of missing entries. We also show the importance of a good initialization for the IRLS algorithm, and our proposed initialization technique using partial distances is shown to work very well. 

\begin{figure}[!t]
	\centering
	\center{\includegraphics[width=0.45\textwidth]{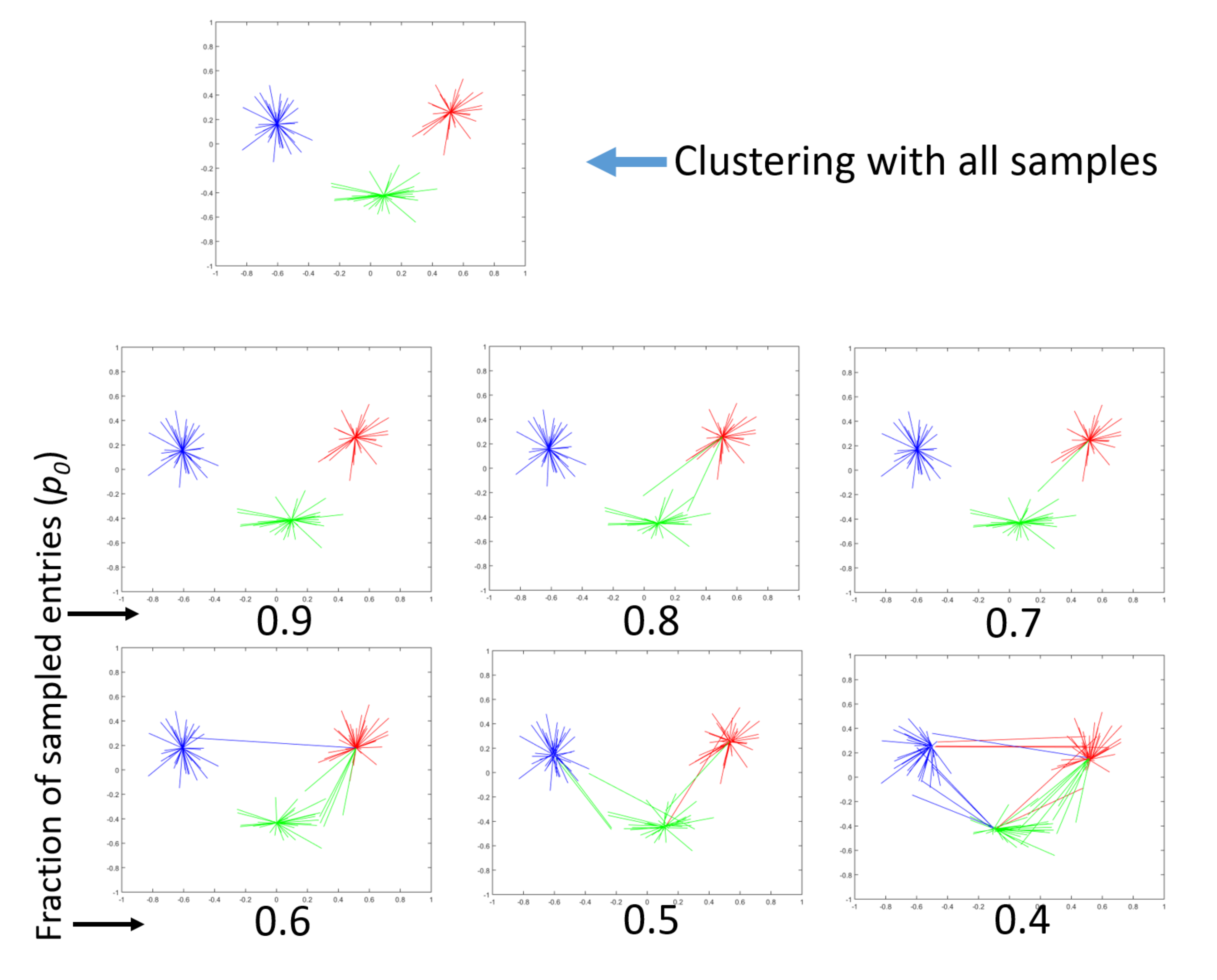}}
	\caption{Clustering on Wine dataset. The $H_1$ penalty is used to cluster the Wine datasets with varying fractions of missing entries. }
	\label{FigWine1}
\end{figure}

\begin{figure}[!t]
	\centering
	\center{\includegraphics[width=0.45\textwidth]{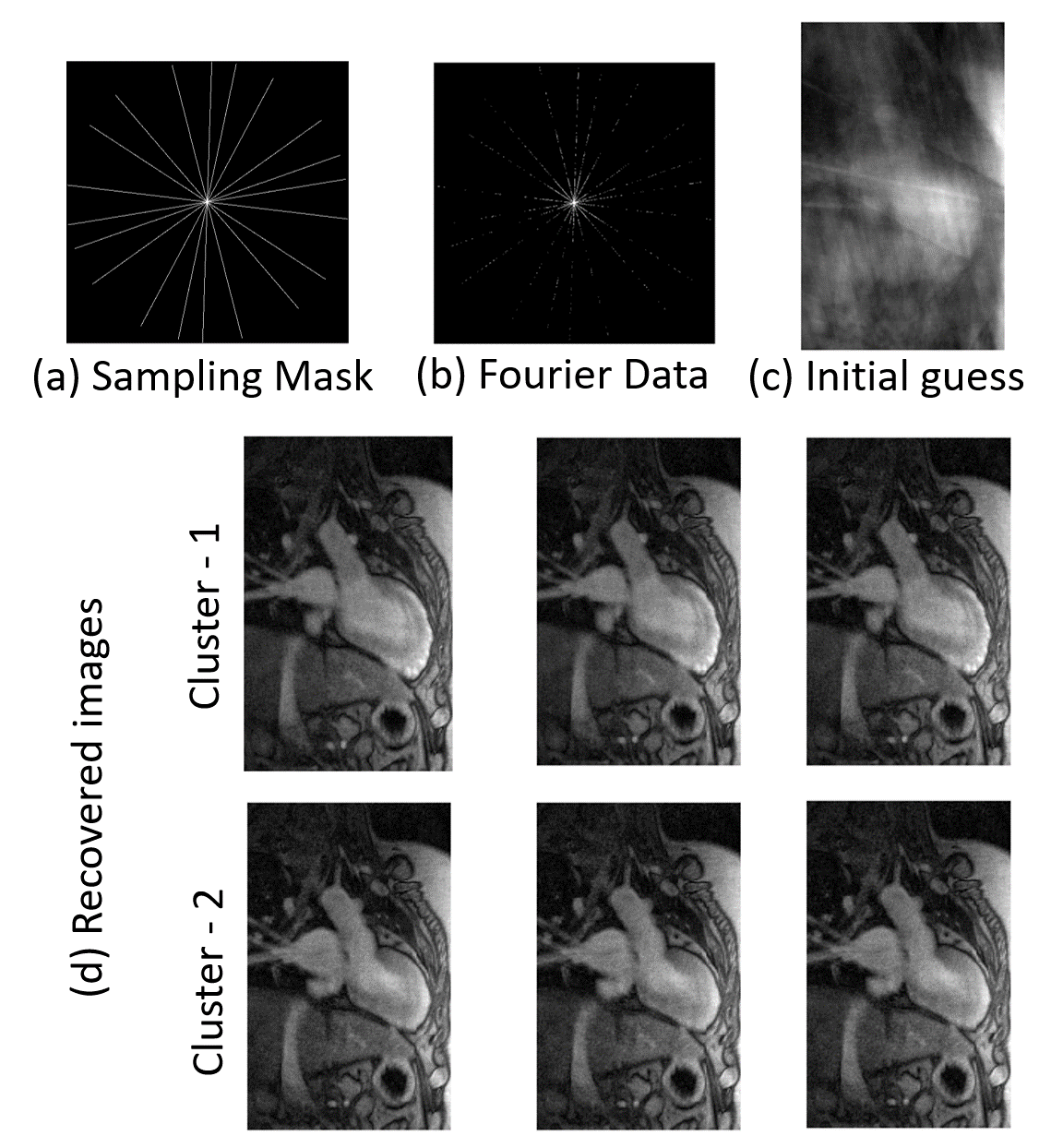}}
	\caption{Cardiac MRI reconstruction results. The  images were reconstructed from highly under-sampled Fourier data using the unconstrained formulation. A sampling mask for 1 particular frame is shown in (a), along with the Fourier data for that frame in (b). The missing Fourier entries were filled with zeros and an inverse Fourier Transform was taken to get the corrupted image in (c). The clustering algorithm was applied to this data and the resulting images were clustered into 20 clusters using spectral clustering. (d) shows some reconstructed images from 2 different clusters.}
	\label{cardiacFig}
\end{figure}

The proposed algorithm performs well on the MR image reconstruction task, resulting in images with minimal artefacts and diagnostic quality. It is to be noted that the MRI images are reconstructed satisfactorily from very few Fourier samples. In this case the fraction of observed samples is around $5\%$. However, we see that the simulated datasets and the Wine datasets cannot be clustered at such a high fraction of missing samples. The fundamental difference between the MRI dataset and the other datasets is the coherence $\mu_0$. For the MRI data, we acquire Fourier samples. Since we know that the low frequency samples are important for image reconstruction, the MRI scanner acquires more low frequency samples. This is a case where high coherence is helpful in clustering. However, for the simulated and Wine data, we do not know apriori which features are more important. In any case the sampling pattern is random, and as predicted by theory, it is more useful to have low coherence. The conclusion is that if the sampling pattern is within our control, it is useful to have high coherence  if the relative importance of the different features is known apriori. If this is unknown, then random sampling is preferred and it is useful to have low coherence. Our future work will focus on deriving guarantees for the case of high $\mu_0$ when the locations of the important features are known with some confidence, and the sampling pattern can be adapted accordingly.

Our theory assumes well-separated clusters and does not consider the presence of any outliers. Theoretical and experimental analysis for the clustering performance in the presence of outliers needs to be investigated. Improving the algorithm performance in the presence of outliers is a direction for future work. Moreover, we have shown improved bounds for the clustering success in the absence of missing entries when the points within a cluster are assumed to follow a uniform random distribution. We expect this trend to also hold for the case with missing entries. This case will be analyzed in future work.

\section{Conclusion}
We propose a clustering technique for data in the presence of missing entries. We prove theoretically that a constrained $\ell_0$ norm minimization problem recovers the clustering correctly even in the presence of missing entries. An efficient algorithm that solves a relaxation of the above problem is presented next. It is demonstrated that the cluster centre estimates obtained using the proposed algorithm degrade gradually with an increase in the number of missing entries. The algorithm is also used to cluster the Wine dataset and reconstruct MRI images from under-sampled Fourier data. The presented theory and results demonstrate the utility of the proposed algorithm in clustering data when some of the feature values of the data are unknown.

\ifCLASSOPTIONcompsoc

\ifCLASSOPTIONcaptionsoff
  \newpage
\fi

\bibliographystyle{IEEEtran}
\bibliography{refs}

\appendices

\section{Proof of Lemma \ref{cent}}
\label{appCent}
\begin{proof} 
	Since $\mathbf x_1 $ and $\mathbf x_2 $ are in the same cluster, $\|\mathbf x_1 - \mathbf x_2\|_{\infty} \leq \epsilon$. For all the points in this particular cluster, let the $p^{th}$ feature be bounded as: $f^p_{min} \leq \mathbf x(p) \leq  f^p_{max}$. Then we can construct a vector $\mathbf u$, such that $\mathbf u(p) = \frac{1}{2}(f^p_{min}+f^p_{max})$. Now, since $f^p_{max}-f^p_{min} \leq \epsilon$, the following condition will be satisfied for this particular choice of $\mathbf u$:
	\begin{eqnarray}
	\|\mathbf x_i-\mathbf u\|_{\infty} &\leq& {\frac{\epsilon}{2}}; ~ ~i=1,2
	\end{eqnarray}
	From this, it follows trivially that the following will also hold:
	\begin{eqnarray}
	\|\mathbf S_i\,(\mathbf x_i-\mathbf u)\|_{\infty} &\leq& {\frac{\epsilon}{2}}; ~ ~i=1,2
	\end{eqnarray}
\end{proof}

\section{Lemma \ref{partialDistBasic}}

\begin{lemma}
	\label{partialDistBasic}
	Consider any pair of points $\mathbf x_1, \mathbf x_2 \in \mathbb R^P$ observed by sampling matrices $\mathbf S_1 = \mathbf S_{\mathcal I_1}$ and $\mathbf S_2 = \mathbf S_{\mathcal I_2}$, respectively. We assume the set of common indices ($\omega \coloneqq \mathcal I_1 \cap \mathcal I_2$) to be of size $q = |\mathcal I_1 \cap \mathcal I_2|$. Then, for some $0< t < \frac{q}{P}$, the following result holds true regarding the partial distance $\|\mathbf y_\omega\|_2 = \|\mathbf S_{\mathcal I_1\cap \mathcal I_2}\left(\mathbf x_{1}-\mathbf x_{2}\right)\|_2$:
	\begin{equation}
	\mathbb P\left(\|\mathbf y_\omega\|_2^2 \leq \left(\frac{q}{P}-t\right)\|\mathbf y\|_2^2\right) \leq e^{-\frac{2t^2P^2}{q\mu_0^2}}
	\end{equation}
\end{lemma}

\begin{proof}
	We use some ideas for bounding partial distances from Lemma 3 of \cite{DBLP:journals/corr/abs-1112-5629}. We rewrite the partial distance $\|\mathbf y_\omega\|_2^2$ as the sum of $q$ variables drawn uniformly at random from $\{y_1^2, y_2^2, \ldots, y_P^2\}$. By replacing a particular variable in the summation by another one, the value of the sum changes by at most $\|\mathbf y\|_{\infty}^2$.
	Applying McDiarmid's Inequality, we get:
	\begin{equation}
	\label{mcdiar}
	\mathbb P\left(E(\|\mathbf y_{\omega}\|_2^2) - \|\mathbf y_\omega\|_2^2 \geq c\right) \leq  e^{-\frac{2c^2}{\sum_{i=1}^{q}\|\mathbf y\|_\infty^4}} =  e^{-\frac{2c^2}{q\|\mathbf y\| _{\infty}^4}}
	\end{equation}
	
	From our assumptions, we have $E(\|\mathbf y_{\omega}\|_2^2) = \frac{q}{P}\|\mathbf y\|_2^2$. We also have $\frac{\|\mathbf y\|_2^2}{\|\mathbf y\| _{\infty}^2} \geq \frac{P}{\mu_0}$ by \eqref{coherence}. We now substitute $c = t\|\mathbf y\|_2^2$, where $0 < t < \frac{q}{P}$. Using the results above, we simplify expression \eqref{mcdiar} as:
	
	\begin{equation}
	\begin {split}
	\mathbb P \left(\|\mathbf y_\omega\|_2^2 \leq \left(\frac{q}{P}-t\right)\|\mathbf y\|_2^2 \right) &\leq e^{-\frac{2t^2\|\mathbf y\|_2^4}{q\|\mathbf y\| _{\infty}^4}}\\
	&\leq e^{-\frac{2t^2P^2}{q\mu_0^2}}\\
	\end {split}
	\end{equation}
	
\end{proof}

\section{Proof of Lemma \ref{partDist}}
\label{lemma2.2proof}
\begin{proof}
	We will use proof by contradiction. Specifically, we consider two points $\mathbf x_1$ and $\mathbf x_2$ belonging to different clusters and assume that there exists a point $\mathbf u$ that satisfies:
	\begin{eqnarray}
	\label{dataconsistencyappendix}
	\|\mathbf S_i\,(\mathbf x_i-\mathbf u)\|_{\infty} &\leq& {\frac{\epsilon}{2}}; i=1,2
	\end{eqnarray}
	We now show that the above assumption is violated with high probability. Following the notation of Lemma \ref{partialDistBasic}, we denote the difference between the vectors by $\mathbf y=\mathbf x_1-\mathbf x_2$ and the partial distances by:
	\begin{equation}
	\|\mathbf y_\omega\|_2 = \|\mathbf S_{\mathcal I_1 \cap \mathcal I_2}~\left(\mathbf x_{1}-\mathbf x_{2}\right)\|_2
	\end{equation}
	Using \eqref{dataconsistencyappendix} and applying triangle inequality, we obtain $
	\|\mathbf y_{\omega}\|_{\infty} \leq {\epsilon}$, which translates to $\|\mathbf y_{\omega}\|_{2} \leq \epsilon\sqrt q$, where $q = |\mathcal I_1 \cap \mathcal I_2|$ is the number of commonly observed locations. We need to show that with high probability, the partial distances satisfy:
	\begin{equation}
	\label{bound}
	\|\mathbf y_{\omega}\|_2^2 > \epsilon^2q
	\end{equation}
	which will contradict \eqref{dataconsistencyappendix}. We first focus on finding a lower bound for $q$. Using the Chernoff bound and setting $\mathbb E(q) = p_0^2\,P$, we have:
	\begin{equation}
	\label{numberofsamples}
	\mathbb P\left(q \geq \frac{p_0^2P}{2}\right) > 1-\gamma_0
	\end{equation}
	where $\gamma_0 = (\frac{e}{2})^{-\frac{p_0^2P}{2}}$. Thus, we can assume that $q \geq \frac{p_0^2P}{2}$ with high probability.
	
	Using Lemma \ref{partialDistBasic}, we have the following result for the partial distances:
	\begin{equation}
	\mathbb P\left(\|\mathbf y_\omega\|_2^2 \leq \left(\frac{q}{P}-t\right)\|\mathbf y\|_2^2\right) \leq e^{-\frac{2t^2P^2}{q\mu_0^2}}
	\end{equation}
	Since $\mathbf x_1$ and $\mathbf x_2$ are in different clusters, we have $\|\mathbf y\|_2 \geq \delta$. We will now determine the value of $t$ for which the above upper bound will equal the RHS of \eqref{bound}:
	\begin{equation}
	\left(\frac{q}{P}-t\right)\|\mathbf y\|_2^2 = \epsilon^2q
	\end{equation}
	or equivalently:
	\begin{equation}
	t=\frac{q}{P}-\frac{\epsilon^2q}{\|\mathbf y\|_2^2} \geq \frac{q}{P}-\frac{\epsilon^2q}{\delta^2} = \frac{q}{P}(1-\kappa^2)
	\end{equation}
	Since $t > 0$, we require $\kappa < 1$, where $\kappa= \frac{\epsilon\sqrt{P}}{\delta}$. Using the above, we get the following bound if we assume that $q \geq \frac{p_0^2P}{2}$:
	\begin{equation}
	\frac{t^2}{q} \geq \frac{q}{P^2}(1-\kappa^2)^2 \geq \frac{p_0^2}{2P}(1-\kappa^2)^2
	\end{equation}
	We now obtain the following probability bound for any $q \geq \frac{p_0^2P}{2}$:  
	\begin{equation}
	\label{highpartial}
	\begin{split}
	\mathbb P\left(\|\mathbf y_\omega\|^2 > \epsilon^2q \right) & \geq  1-e^{-\frac{2t^2P^2}{q\mu_0^2}}\\ & \geq 1-e^{-\frac{p_0^2P(1-\kappa^2)^2}{\mu_0^2}}\\ & = 1- \delta_0
	\end{split}
	\end{equation}
	Combining \eqref{numberofsamples} and \eqref{highpartial}, the probability for \eqref{dataconsistencyappendix} to hold is $\leq 1 - (1-\gamma_0)(1-\delta_0) = \beta_0$.
	
\end{proof}

\section{Proof of Lemma \ref{smallClusters}}
\label{appSC}
\begin{proof}
	
	We construct a graph where each point $\mathbf x_i$ is represented by a node. Lemma \ref{cent} implies that a pair of points belonging to the same cluster can yield the same $\mathbf u$ in a feasible solution with probability 1. Hence, we will assume that there exists an edge between two nodes from the same cluster with probability 1. Lemma \ref{partDist} indicates that a pair of points belonging to different clusters can yield the same $\mathbf u$ in a feasible solution with a low probability of $\beta_0$. We will assume that there exists an edge between two nodes from different clusters with probability $\beta_0$. We will now evaluate the probability that there exists a fully-connected sub-graph of size $M$, where all the nodes have not been taken from the same cluster. We will follow a methodology similar to \cite{matula1976largest}, which gives an expression for the probability distribution of the maximal clique (i.e. largest fully connected sub-graph) size in a random graph. Unlike the proof in \cite{matula1976largest}, in our graph every edge is not present with equal probability. 
	
	We define the following random variables:
	\begin{itemize}
		\item $t \coloneqq$ Size of the largest fully connected sub-graph containing nodes from more than $1$ cluster
		\item $n \coloneqq$ Number of $M$ membered complete sub-graphs containing nodes from more than $1$ cluster
	\end{itemize}
	Our graph can have an $M$ membered clique iff $n$ is non-zero. Thus, we have:
	\begin{equation}
	\mathbb P \left(t \geq M \right) = \mathbb P \left(n \neq 0 \right)
	\end{equation}
	Since the distribution of $n$ is restricted only to the non-negative integers, it can be seen that:
	\begin{equation}
	\mathbb P \left(n\neq 0\right) \leq E(n)
	\end{equation}
	Combining the above 2 results, we get:
	\begin{equation}
	\mathbb P \left(t \geq M \right) \leq E(n)
	\end{equation}
	
	Let us consider the formation of a particular clique of size $M$ using $m_1, m_2, \ldots, m_K$ nodes from clusters $C_1, C_2, \ldots, C_K$ respectively such that $\sum_{j=1}^{K}m_j = M$, and at least $2$ of the variables $\{m_j\}$ are non-zero. The number of ways to choose such a collection of nodes is: $\prod_j {M \choose m_j}$. In order to form a solution $\{m_j\}$, we need $\frac{1}{2}(M^2-\sum_j {m_j^2})$ inter-cluster edges to be present. We recall that each of these edges is present with probability $\beta_0$. Thus, the probability that such a collection of nodes forms a clique is $\beta_0^{\frac{1}{2}(M^2-\sum_j {m_j^2})}$. This gives the following result:
	\begin{equation}
	E(N) = \sum_{\{m_j\} \in \mathcal S} \beta_0^{\frac{1}{2}(M^2-\sum_j {m_j^2})} \prod_j {M \choose m_j} = \eta_0
	\end{equation}
	where $\mathcal S$ is the set of all sets of positive integers $\{m_j\}$ such that: $2 \leq \mathcal U(\{m_j\}) \leq K$ and $\sum_j m_j = M$. Here, the function $\mathcal U$ counts the number of non-zero elements in a set. Thus, we have:
	\begin{equation}
	\mathbb P \left(t\geq M \right) \leq \eta_0
	\end{equation}
	This proves that with probability $\geq 1-\eta_0$, a set of points of cardinality $\geq M$ not all belonging to the same cluster cannot all have equal cluster-centre estimates.
	
\end{proof}

\section{Proof of Theorem \ref{mainResult}}
\label{appMR}
\begin{proof}
	Lemma \ref{cent} indicates that fully connected original clusters with size $M$ are likely with probability 1, while Lemma  \ref{smallClusters} shows that the size of misclassified large clusters cannot exceed $M-1$ with very high probability. These results enable us to re-express the optimization problem \eqref{l0prob} as a simpler maximization problem. We will then show that with high probability, any feasible solution other than the ground-truth solution results in a cost higher than the ground-truth solution.
	
	Let a candidate solution have $k$ groups of sizes $M_1, M_2,\ldots, M_k$ respectively. The centre estimates for all points within a group are equal. These are different from the centre estimates of other groups. Without loss of generality, we will assume that at most $K$ of these groups each have points belonging to only a single ground-truth cluster, i.e. they are "pure". The rest of the clusters in the candidate solution are "mixed" clusters. If we have a candidate solution with greater than $K$ pure clusters, then they can always be merged to form $K$ pure clusters; the merged solution will always result in a lower cost. 
	
	The objective function in \eqref{l0prob} can thus be rewritten as:
	\begin{equation}
	\begin{split}
	\sum_{i=1}^{KM}\sum_{j=1}^{KM}\|\mathbf u_i - \mathbf u_j\|_{2,0} 
	& = \sum_{i=1}^{k} M_i (KM-M_i) \\
	& = K^2M^2 - \sum_{i=1}^{k} M_i^2 
	\end{split}
	\end{equation}
	Since we assume that the first $K$ clusters are pure, therefore they have a size $0 \leq M_i \leq M$, $i=1,\ldots, K$. The remaining clusters are mixed and have size $\leq M-1$ with probability $\geq 1-\eta_0$. Hence, we have the constraints $0 \leq M_i \leq (M-1)$, $i=K+1,\ldots, k$. We also have a constraint on the total number of points, i.e. $\sum_{i=1}^k M_i = KM$. Thus, the problem \eqref{l0prob} can be rewritten as the constrained optimization problem:
	\begin{equation}
	\begin{split}
	\{M_i^*,k^*\} = & \max_{\{M_i\},k}\sum_{i=1}^{k} M_i^2 \\
	\mbox{ s.t. } & 0 \leq M_i \leq M, i=1,\ldots, K \\
	& 0 \leq M_i \leq M-1, i=K+1,\ldots,k \\
	& \sum_{i=1}^k M_i = KM
	\end{split}
	\end{equation}
	Note that we cannot have $k < K$, with probability $\geq 1-\eta_0$, since that involves a solution with cluster size $> M$. We can evaluate the best solution $\{M_i^*\}$ for each possible value of $k$ in the range $K \leq k \leq MK$. Then we can compare these solutions to get the solution with the highest cost. We note that the feasible region is a polyhedron and the objective function is convex. Thus, for each value of $k$, we only need to check the cost at the vertices of the polyhedron formed by the constraints, since the cost at all other points in the feasible region will be lower. The vertex points are formed by picking $k-1$ out of the $k$ box constraints and setting $M_i$ to be equal to one of the 2 possible extremal values. We note that all the vertex points have either $K$ or $K+1$ non-zero values. As a simple example, if we choose $M=10$ and $K=4$, then the vertex points of the polyhedron (corresponding to different solutions $\{M_i\})$ are given by all possible permutations of the following:
	
	\begin{itemize}
		\item $(10,10,10,10,0,0 \ldots 0)$ : 4 clusters
		\item $(10,10,10,0,1,9,0 \ldots 0)$: 5 clusters
		\item $(10,10,0,0,2,9,9,0 \ldots 0)$: 5 clusters
		\item $(10,0,0,0,3,9,9,9,0 \ldots 0)$: 5 clusters
		\item $(0,0,0,0,4,9,9,9,9,0 \ldots 0)$: 5 clusters
	\end{itemize} 
	In the general case the vertices are given by permutations of the following:
	\begin{itemize}
		\item $(M,M,\ldots,M,0,0 \ldots 0)$: $K$ clusters
		\item $(M,M,\ldots,0,0,1,M-1,0 \ldots 0)$: $K+1$ clusters
		\item $(M,M,\ldots,0,0,2,M-1,M-1 \ldots 0)$: $K+1$ clusters
		\item \ldots
		\item $(0,0, \ldots 0,K,M-1,M-1 \ldots M-1, 0)$: $K+1$ clusters
	\end{itemize} 
	
	Now, it is easily checked that the $1^{st}$ candidate solution in the list (which is also the ground-truth solution) has the maximum cost. Mixed clusters with size $> M-1$ cannot be formed with probability $> 1 - \eta_0$. Thus, with the same probability, the solution to the optimization problem \eqref{l0prob} is identical to the ground-truth clustering. This concludes the proof of the theorem.
	
\end{proof}

\section{Upper Bound for $\eta_0$ in the 2-cluster case}
\label{appUB}
\begin{proof}
	
	We introduce the following notation: 
	\begin{enumerate}
		\item $F(i) = i(M-i)\log \beta_0$, for $i \in [1, M-1]$.
		\item $G(i) = 2[\log \Gamma (M+1) - \log \Gamma(i+1) - \log \Gamma(M-i+1)]$, for $i \in [1, M-1]$ where $\Gamma$ is the Gamma function.
	\end{enumerate}
	We note that both the functions $F$ and $G$ are symmetric about $i = \frac{M}{2}$, and have unique minimum and maximum respectively for $i = \frac{M}{2}$. We will show that the maximum for the function $F + G$ is achieved at the points $i=1,M-1$. We note that:
	\begin{equation}
	G'(i) = -2[\Psi(i+1) - \Psi(M-i+1)]
	\end{equation}
	where $\Psi$ is the digamma function, defined as the log derivative of the $\Gamma$ function. We now use the expansion:
	\begin{equation}
	\Psi(i+1) = \log i + \frac{1}{2i}
	\end{equation}
	Substituting, we get:
	\begin{equation}
	G'(i) = - 2\left[\log \frac{i}{M-i} + \frac{M-2i}{2i(M-i)}\right]
	\end{equation}
	We also have:
	\begin{equation}
	F'(i) = (M-2i)\log \beta_0
	\end{equation}
	Adding, we get:
	\begin{equation}
	\begin{split}
	F'(i) + G'(i) = (M-2i) (&\log \beta_0 - \frac{1}{i(M-i)})\\& -2\log\frac{i}{(M-i)})
	\end{split}
	\end{equation}
	Now, in order to ensure that $F'(i) + G'(i) \leq 0$, we have to arrive at conditions such that:
	\begin{equation}
	\log \beta_0 \leq \frac{1}{i(M-i)} + \frac{2}{M-2i}\log \frac{i}{M-i}
	\end{equation}
	Since the RHS is monotonically increasing in the interval $i \in [1, \frac{M}{2}-1]$ the above condition reduces to:
	\begin{equation}
	\log \beta_0 \leq \frac{1}{M-1} + \frac{2}{M-2}\log \frac{1}{M-1}
	\end{equation}
	Under the above condition, for all $i \in [1, \frac{M}{2}]$ :
	\begin{equation}
	F'(i) + G'(i) \leq 0
	\end{equation}
	Thus, the function $F + G$ reaches its maxima at the extremal points given by $i=1,M-1$. 
	For positive integer values of $i$, i.e. $i \in \{1, 2, \ldots, M-1\}$:
	\begin{equation}
	F(i) + G(i) = \log[\beta_0^{i(M-i)} {M \choose i}^2]
	\end{equation}
	Thus, the function $\beta_0^{i(M-i)} {M \choose i}^2$ also reaches its maxima at $i=1,M-1$. This maximum value is given by: $\beta_0^{M-1}M^2$. This gives the following upper bound for $\eta_0$:
	\begin{equation}
	\begin{split}
	\eta_0 & \leq \sum_{i=1}^{M-1} [\beta_0^{M-1} M^2] \\ & = M^2(M-1)\beta_0^{M-1} \\ & \leq M^3 \beta_0^{M-1} \\ & = \eta_{0,{\rm approx}}
	\end{split}
	\end{equation}
	
\end{proof}

\section{Proof of Theorem \ref{noMissingFinal}}
\label{appNM}
\begin{proof}
	We consider any two points $\mathbf x_1$ and $\mathbf x_2$ that are in different clusters. Let us assume that there exists some $\mathbf u$ satisfying  the data consistency constraint:
	\begin{equation}
	\|\mathbf x_i - \mathbf u\|_{\infty} \leq \epsilon/2, ~~i=1,2.
	\end{equation}
	Using the triangle inequality, we have $\|\mathbf x_1 - \mathbf x_2\|_{\infty} \leq \epsilon$ and consequently, $\|\mathbf x_1 - \mathbf x_2\|_{2} \leq \epsilon\sqrt P$. However, if we have a large inter-cluster separation $\delta > \epsilon\sqrt P$, then this is not possible.
	
	Thus, if $\delta > \epsilon\sqrt P$, then points in different clusters cannot be misclassified to a single cluster. Among all feasible solutions, clearly the solution to problem \eqref{l0probFull} with the minimum cost is the one where all points in the same cluster merge to the same $\mathbf u$. Thus, $\kappa < 1$ ensures that we will have the correct clustering.
\end{proof}

\section{Proof of Lemma \ref{noMissingUniform}}
\label{appNMU}
\begin{proof}
	
	The idea is similar to that in Theorem \ref{noMissingFinal}. We will show that with high probability two points $\mathbf x_1$ and $\mathbf x_2$ that are in different clusters satisfy $\|\mathbf x_1 - \mathbf x_2\|_2 > \epsilon \sqrt{P}$ with high probability, which implies that \eqref{dataconsistencynoise} is violated.
	
	Let points in $C_1$ and $C_2$ follow uniform random distributions in $\mathbb R^P$ with centres $\mathbf c_1$ and $\mathbf c_2$ respectively. The expected distance between $\mathbf x_1 \in \mathcal C_1$ and $\mathbf x_2 \in \mathcal C_2$ is given by:
	
	\begin{equation}
	\begin{split}
	E(\|\mathbf x_1 - \mathbf x_2\|_2^2) &= \frac{1}{\epsilon^{2}}\sum_{p=1}^P\int_{\mathbf c_1^p-\frac{\epsilon}{2}}^{\mathbf c_1^p+\frac{\epsilon}{2}}\int_{\mathbf c_2^p-\frac{\epsilon}{2}}^{\mathbf c_2^p+\frac{\epsilon}{2}}(\mathbf x_1^p - \mathbf x_2^p )^2 d\mathbf x_1^p d\mathbf x_2^p \\ &= \|\mathbf c_1 - \mathbf c_2\|_2^2 + \frac{P}{6}\epsilon^2\\
	&= c_{12}^2 + \frac{P}{6}\epsilon^2
	\end{split}
	\end{equation} 
	where $\mathbf c_i^p$ and $\mathbf x_i^p$ are the $p^{th}$ features of $\mathbf c_i$ and $\mathbf x_i$ respectively, and $c_12 = \|\mathbf c_1 - \mathbf c_2\|_2$.
	Let $c_i = |\mathbf c_1^i - \mathbf c_2^i|$, for $i = 1,2,\ldots,P$. Using Mcdiarmid's inequality:
	\begin{equation}
	\begin{split}
	& \mathbb P \left(\|\mathbf x_1 - \mathbf x_2\|_2^2 \leq E(\|\mathbf x_1 - \mathbf x_2\|_2^2) - t \right) \\ & \leq e^{-\frac{2t^2}{\sum_{i=1}^P|(c_i+\epsilon)^2-(c_i-\epsilon)^2|^2}} \\ & = e^{-\frac{t^2}{8\epsilon^2c_{12}^2}}
	\end{split}
	\end{equation}
	Let $t = E(\|\mathbf x_1 - \mathbf x_2\|_2^2) - P\epsilon^2$. Then we have:
	\begin{equation}
	\mathbb P \left(\|\mathbf x_1 - \mathbf x_2\|_2 \leq \epsilon \sqrt{P} \right) \leq e^{-\frac{(c_{12}^2-\frac{5P}{6}\epsilon^2)^2}{8\epsilon^2c_{12}^2}}
	\end{equation}
	We note that the RHS above is a decreasing function of $c_{12}$. Thus, we consider some $c \leq c_{12}$, such that $c$ is the minimum distance between any $2$ cluster centres in the dataset. We then have the following bound:
	\begin{equation}
	\mathbb P \left(\|\mathbf x_1 - \mathbf x_2\|_2 \leq \epsilon \sqrt{P} \right) \leq e^{-\frac{(c^2-\frac{5P}{6}\epsilon^2)^2}{8\epsilon^2c^2}}
	\end{equation}
	To ensure $t>0$, we require: $c > \sqrt{\frac{5P}{6}}\epsilon$, or equivalently, $\kappa' = \frac{\epsilon \sqrt{P}}{c} < \sqrt{\frac{6}{5}}$.
	
	We now get the probability bound:
	\begin{equation}
	\mathbb P \left(\|\mathbf x_1 - \mathbf x_2\|_2 \leq \epsilon \sqrt{P}\right) \leq e^{-\frac{P(1-\frac{5}{6}\kappa'^{2})^2}{8 \kappa'^{2}}} = \beta_1
	\end{equation}
	Thus, \eqref{dataconsistencynoise} is violated with probability exceeding $1-\beta_1$.
\end{proof}

\end{document}